\documentclass[twoside]{article}

\usepackage[preprint]{aistats2026}
\usepackage[
backend=biber,
style=authoryear,
natbib=true,
maxnames=2
]{biblatex}
\usepackage{our-style}
\usepackage{xfrac}
\begin{document}

%

%

\twocolumn[

\aistatstitle{A Unifying Framework for Unsupervised Concept Extraction}

\aistatsauthor{ Chandler Squires \And Pradeep Ravikumar }

\aistatsaddress{ Machine Learning Department, CMU \And Machine Learning Department, CMU} ]

\begin{abstract}
    Techniques for \emph{concept extraction}, such as sparse autoencoders and transcoders, aim to extract high-level symbolic concepts from low-level nonsymbolic representations.
    When these extracted concepts are used for downstream tasks such as model steering and unlearning, it is essential to understand their guarantees, or lack thereof.
    In this work, we present a unified theoretical framework for unsupervised concept extraction, in which we frame the task of concept extraction as identifying a generative model.
    We present a general meta-theorem for identifiability, which reduces the problem of establishing identifiability guarantees to the problem of characterizing the intersection of two sets.
    As we demonstrate on a range of widely-used approaches, this meta-theorem substantially simplifies the task of proving such guarantees, thus paving the way for the development of new, principled approaches for concept extraction.
    
\end{abstract}

\addtocontents{toc}{\protect\setcounter{tocdepth}{0}}

\section{INTRODUCTION}\label{sec:intro}

As large neural networks have become the de facto standard for many machine learning tasks, there has been growing interest in understanding the internal representations that they learn.
Such an understanding promises to enable several downstream capabilities, e.g. intervening on these representations can be an effective approach for post-training tasks such as model steering, unlearning, and representation alignment \citep{turner2023activation,ashuach2025crisp,sucholutsky2023getting}.

A key challenge to gaining such an understanding is that neural representations are often \emph{entangled} and \emph{distributed}, rather than \emph{disentangled} and \emph{local}.
More formally, consider a pre-trained encoder $\Phi\colon \cX \to \cZ$ mapping from a high-dimensional input space $\cX$ to a neural \emph{feature space} $\cZ$.
For example, $\Phi$ may be the first 10 layers of an image classification model, which maps an image $x \in \cX$ to a neural activation vector $z = \Phi(x) \in \cZ$.
Typically, the individual coordinates of $z$ are \emph{polysemantic}, i.e., each coordinate $z_i$ represents a combination of many human-interpretable features rather than a single one \citep{elhage2022toy}.

To address this challenge, a common approach is to extract high-level \emph{concepts} from these neural representations, e.g. by training a sparse autoencoder to reconstruct the neural representations \citep{cunningham2023sparse}.
Such approaches specify a \emph{concept space} $\cC$ and find some \emph{concept extractor} $h\colon \cZ \to \cC$.
The concept extractor is typically taken to be injective, so that an intervention on the concepts (i.e., a function $f_c\colon \cC \to \cC$) can be translated into a corresponding intervention $f_z = h^{-1} \circ f_c \circ h$ on the neural features.

Unfortunately, many existing approaches for concept extraction do not reliably recover a ``canonical" set of features, e.g. for different random seeds, the same sparse autoencoder can recover very different concept extractors $h$ and $h'$ from the same data \citep{leask2025sparse,paulo2025sparse}.
This ambiguity both hinders interpretability and casts doubt on the prospect of using the extracted concepts for downstream tasks, since the same concept-level intervention will result in different feature-level interventions \citep{meloux2025everything,song2025position}.
To avoid these issues, it is essential to treat the task of concept extraction more rigorously, and develop new methods with appropriate uniqueness guarantees.

In this work, we consider the task of concept extraction from a rigorous statistical viewpoint, using the tools of \emph{identifiability theory}.
We treat concepts as latent random variables $\bC$ sampled from some unknown distribution $Q$ belonging to some known class $\cQ$, and neural features as observed random variables $\bZ$ which depend on $\bC$ in an unknown fashion.

In the deterministic setting (i.e., $\bZ = g(\bC)$ for some unknown, injective $g$), the task of concept extraction is thus equivalent to identifying $h = g^{-1}$ up to simple symmetries, e.g. relabelings of the concepts.
In this setting, the choice of class $\cQ$ can be seen as an \textit{explicit} description of the \textit{implicit} assumptions imposed by different sparse autoencoders \citep{hindupur2025projecting}.
More generally, our framework permits \emph{stochastic} relationships between concepts and features, opening the door to a variety of more realistic cases.

\paragraph{Contributions}
Our contributions are four-fold:
\begin{enumerate}[itemsep=1pt,topsep=1pt]
    \item We introduce a new theoretical framework for unsupervised concept extraction based on \emph{latent concept generative models (LC-GMs)}, which are models that can be decomposed into a \emph{concept generator} $Q$ and a \emph{mixing kernel} $K$ (\Cref{sec:setup}).

    \item We present general meta-theorems for the identifiability of a class $\cM$ of LC-GMs (\Cref{sec:results}).
    Most importantly, our \emph{intersection theorem} (\Cref{thm:intersection}) characterizes identifiability via the intersection of two \emph{valid transition sets}, extending prior work by \citet{xi2023indeterminacy} to model classes which allow for stochastic mixing.

    \item To highlight the meta-theoretical utility of our framework, we show how our theorems can recover diverse existing identifiability results, e.g. in dictionary learning, independent component analysis, and mixture modeling (\Cref{sec:applications}, \Cref{appendix:finite-mixture-models}).

    \item To highlight the potential practical utility of our framework, we discuss its implications for method development, including connections between concept extraction, generative modeling, and amortized posterior inference (\Cref{sec:methodological}).
    
\end{enumerate}


\section{NOTATION}\label{sec:notation}

As is common in identifiability theory, our results are given in the \emph{population setting}, i.e., we work directly with probability distributions rather than random variables and samples.
As such, the mathematical context for this work is measure-theoretic probability, see \Cref{appendix:background} for key definitions and references.

In this paper, we assume all spaces are \defword{Polish Borel spaces}.
These spaces enjoy several useful properties, which we introduce as needed.
Two basic examples include the sets $\bbR^d$ and $\{ 1, 2, \ldots, d \}$ with their standard topologies and $\sigma$-algebras.
In this section, we let $\cU$, $\cV$, and $\cW$ be Polish Borel spaces.

\paragraph{Distributions}

For any measure $\mu$ on $\cV$ (including probability distributions), we define the \defword{support} of $\mu$, denoted $\supp(\mu)$, as the smallest closed set $F \subseteq \cV$ such that $\mu(\cV \setminus F) = 0$.
We let $\cP(\cV)$ denote the set of probability distributions on $\cV$.
We say that $Q$ is \defword{absolutely continuous} with respect to a measure $\mu$ if $\mu(A) = 0$ implies $Q(A) = 0$ for any measurable set $A$, written $Q \ll \mu$.
If $Q \ll \mu$ and $\mu \ll Q$, we write $Q \sim \mu$.
Finally, given a point $v \in \cV$, we let $\dirac_v \in \cP(\cV)$ denote the delta distribution at $v$.


\paragraph{Markov Kernels}
We let $\cK(\cU \to \cV)$ denote the set of Markov kernels from $\cU$ to $\cV$.
For a Markov kernel $K$ from $\cU$ to $\cV$, we let $K(\cdot \mid u)$ denote the associated probability distribution under $u \in \cU$.
Given two Markov kernels $K, K' \in \cK(\cU \to \cV)$ and a measure $\mu$ on $\cU$, we write $K \sim_\mu K'$ to denote that $K$ and $K'$ are equal $\mu$-almost everywhere (abbreviated $\mu$-a.e.).

For $K \in \cK(\cU \to \cV)$, we define the \defword{pushforward operator} $K_\sharp\colon \cP(\cU) \to \cP(\cV)$ as follows:
\begin{align*}
    K_\sharp&\colon Q \mapsto
    \textstyle{\int}_\cU K(\cdot \mid u) \cdot Q(du).
\end{align*}
A \defword{posterior kernel} of $K$ with respect to $Q$ is a Markov kernel $L \in \cK(\cV \to \cU)$ which satisfies the condition $K(dv \mid du) \cdot Q(du) = L(du \mid v) \cdot P(dv)$, where $P = K_\sharp(Q)$.
In Polish Borel spaces, a posterior kernel always exists and is unique $P$-a.e.; we denote any such kernel as $K^\dagger_Q$.

We say that $K$ is \defword{measure-separating} if its associated pushforward operator $K_\sharp$ is injective.
We let $\cK_\sep(\cU \to \cV)$ denote the set of measure-separating Markov kernels from $\cU$ to $\cV$.\footnote{More generally, we say $K$ is \defword{measure-separating on $\cP' \subseteq \cP$} if the restriction $K_\sharp|_{\cP'}$ is injective, and we denote the set of such kernels as $\cK_\sep(\cU \to \cV ; \cP')$.}
As a special case, we let $\MeasInj(\cU \to \cV)$ denote the set of injective measurable functions from $\cU$ to $\cV$, and we let $\ContMeasInj(\cU \to \cV)$ denote the set of continuous injective measurable functions from $\cU$ to $\cV$.
For a measurable function $g\colon \cU \to \cV$, we let $\Dirac_g \in \cK(\cU \to \cV)$ denote the \defword{Dirac kernel} associated with $g$, i.e., $\Dirac_g\colon u \mapsto \dirac_{g(u)}$.

Given two Markov kernels $K \in \cK(\cU \to \cV)$ and $L \in \cK(\cV \to \cW)$, we let $LK \in \cK(\cU \to \cW)$ denote their composition, defined as
\begin{equation*}
    (LK)\colon u \mapsto \textstyle{\int}_{\cV} L(\cdot \mid v) \cdot K(dv \mid u).
\end{equation*}
For Polish Borel spaces, pushforward and composition can be interchanged, i.e., $(LK)_\sharp = L_\sharp \circ K_\sharp$.
We refer to this property as \defword{functoriality}.
Since a probability distribution $Q \in \cP(\cU)$ can be treated as Markov kernel from the empty set to $\cU$, we also use the notation $KQ = K_\sharp(Q)$.
In particular, by functoriality, $(LK)(Q) = L(KQ)$.

\paragraph{Composition Notation}
For two sets of Markov kernels $\cK \subseteq \cK(\cU \to \cV)$ and $\cL \subseteq \cK(\cV \to \cW)$, we let
\begin{equation*}
    \cL \circ \cK \defeq \{ LK : L \in \cL, K \in \cK \}.
\end{equation*}
In \Cref{appendix:background}, we give an extension of this notation when $\cL$ or $\cK$ are replaced by sets of functions, or when $\cK$ is replaced by a set of probability distributions.

\paragraph{Matrix-Vector Notation}
For several examples, we use discrete spaces $\cU$ and $\cV$.
In such settings, distributions and Markov kernels can be conveniently represented by (labelled) vectors and matrices, respectively.
In this paper, we use column-stochastic convention, i.e., distributions are column vectors, and Markov kernels are column-stochastic matrices.
For example, consider $\cU = \{ a, b \}$ and $\cV = \{ u, v \}$.
Then
\vspace{-10pt}
\begin{equation*}
    Q = \bordermatrix{
    & \cr
    \mylabel{a} & \beta_a \cr
    \mylabel{b} & \beta_b \cr
    }\quad\text{and}\quad
    K = \bordermatrix{
      & \mylabel{a} & \mylabel{b} \cr
    \mylabel{u} & \beta_{au} & \beta_{bu} \cr
    \mylabel{v} & \beta_{av} & \beta_{bv} \cr
    }
\end{equation*}
are the distribution $Q = \beta_a \cdot \delta_a + \beta_b \cdot \delta_b$ and the Markov kernel $K$ with $K(\cdot \mid a) = \beta_{au} \cdot \delta_u + \beta_{av} \cdot \delta_v$ and $K(\cdot \mid b) = \beta_{bu} \cdot \delta_u + \beta_{bv} \cdot \delta_v$, respectively.



\section{LATENT CONCEPT GENERATIVE MODELS}\label{sec:setup}

Let $\cC$ and $\cZ$ be Polish Borel spaces, which we call a \defword{concept space} and \defword{feature space}, respectively.
Though we typically consider $\cC \subseteq \bbR^d$ or $\cC = \{ 1, 2, \ldots, d \}$, we note that the Polish Borel condition is quite weak.
Thus, our results can be applied to many complex spaces, e.g. some function spaces or spaces of structured objects such as graphs.
A \defword{concept distribution} is a distribution over $\cC$, and a \defword{mixing kernel} is a Markov kernel from $\cC$ to $\cZ$.

\subsection{Basic Definitions}
As described in \Cref{sec:intro}, we consider the task of concept extraction from a statistical viewpoint, in which the concepts are latent random variables which generate the observed neural features.
We formalize this generation process as a \defword{latent concept generative model (LC-GM)} $\M$ from $\cC$ to $\cZ$, which is a tuple $\M = (Q, K)$ with $Q \in \cP(\cC)$ and $K \in \cK(\cC \to \cZ)$.\footnote{
We note that LC-GMs are deep latent variable models (\textit{c.f.} \citet[Section 1.7]{diederik2019introduction}) where \emph{both} the ``prior" $Q$ and the ``stochastic decoder" $K$ are unknown.
}

Given a model $\M = (Q, K)$, the \defword{induced feature distribution} of $\M$, denoted $P^\M \in \cP(\cZ)$, is
\begin{equation*}
    P^\M 
    \defeq 
    KQ
\end{equation*}
i.e., $P^\M$ represents the marginal distribution over features $\bZ$ when first sampling concepts $\bC$ from $Q$, then sampling $\bZ$ from the distribution $K(\cdot \mid \bC)$.

Similarly, the \defword{induced concept extractor}, denoted $H^\M \in \cK(\cZ \to \cC)$, is defined as $H^\M = K^\dagger_Q$, where $K^\dagger_Q$ is the posterior kernel of $K$ with respect to $Q$.
We illustrate these two concepts with an example:


\begin{example}[Induced objects]\label{example:ifd}
    Consider the concept space $\cC = \bbR^d$ and the feature space $\cZ = \bbR^p$.

    For a left-invertible matrix $\matG \in \bbR^{p \times d}$, let $\M = (Q, K)$, where $Q = \cN(\bzero, \matI_d)$ and $K\colon \bc \mapsto \delta_{\matG \cdot \bc}$.
    The induced feature distribution of $\M$ is $P^\M = \cN(\bzero, \matG \cdot \matG^\top)$, and the induced concept extractor is $H^\M\colon \bz \mapsto \delta_{\matH \cdot \bz}$, where $\matH \in \bbR^{d \times p}$ is the left inverse of $\matG$.
\end{example}

We let $\cM_\all(\cC, \cZ)$ denote the set of latent concept generative models from $\cC$ to $\cZ$, i.e., 
\begin{equation*}
    \cM_\all(\cC, \cZ) \defeq \cP(\cC) \times \cK(\cC \to \cZ).
\end{equation*}
Throughout this section, let $\cC$ and $\cC'$ be concept spaces, let $\M = (Q, K)$ be a LC-GM from $\cC$ to $\cZ$, and let $\M' = (Q', K')$ be a LC-GM from $\cC'$ to $\cZ$.

\subsection{Feature Equivalence}
When performing concept extraction in the population setting, we observe a distribution $P^* \in \cP(\cZ)$ over features, and aim to find some LC-GM $\M$ such that $P^\M = P^*$.
Intuitively, such a model $\M$ represents a possible explanation for the observed distribution $P$, and a possible concept extractor $H^\M$.

However, multiple LC-GMs can induce the same feature distribution, but induce different concept extractors.
In particular, we say $\M$ and $\M'$ are \defword{feature equivalent} if $P^\M = P^{\M'}$.
As the following simple example shows, feature equivalence does not imply $H^\M = H^{\M'}$, even up to relabeling.

\begin{example}[Feature equivalence]\label{ex:feature-equivalence}
    Consider the spaces $\cC = \{ a, b \}$, $\cC' = \{ c, d \}$, and $\cZ = \{ u, v \}$.
    Let $\M = (Q, K)$, and let $\M' = (Q', K')$, for
    \vspace{-10pt}
    \begin{equation*}
    \begin{aligned}
        K &= \bordermatrix{
          & \mylabel{a} & \mylabel{b} \cr
        \mylabel{u} & \sfrac{2}{3} & \sfrac{1}{3} \cr
        \mylabel{v} & 0 & 1 \cr
        },
        &\qquad
        Q &= \bordermatrix{
        & \cr
        \mylabel{a} & \sfrac{3}{4}  \cr
        \mylabel{b} & \sfrac{1}{4}
        },
        \\
        K' &= \bordermatrix{
          & \mylabel{c} & \mylabel{d} \cr
        \mylabel{u} & 1 & 0 \cr
        \mylabel{v} & 0 & 1 \cr
        },
        &\qquad
        Q' &= \bordermatrix{
        & \cr
        \mylabel{c} & \sfrac{1}{2}  \cr
        \mylabel{d} & \sfrac{1}{2}
        }.
    \end{aligned}
    \end{equation*}
    Then $P^\M = P^{\M'} = \frac{1}{2}(\delta_u + \delta_v)$.
    However, $H^\M \neq H^{\M'}$:
    \vspace{-10pt}
    \begin{equation*}
        H^\M = \bordermatrix{
            & \mylabel{u} & \mylabel{v} \cr
          \mylabel{a} & 1 & \sfrac{1}{2} \cr
          \mylabel{b} & 0 & \sfrac{1}{2}
        }
        \quad
        H^{\M'} = \bordermatrix{
            & \mylabel{u} & \mylabel{v} \cr
          \mylabel{c} & 1 & 0 \cr
          \mylabel{d} & 0 & 1
        }.
    \end{equation*}
\end{example}

\subsection{Blackwell Coarsening and Equivalence}\label{sec:blackwell-coarsening}
In \Cref{sec:results}, we will consider restricted model classes $\cM \subseteq \cM_\all(\cC, \cZ)$, and our main interest will be developing \emph{necessary} conditions for feature equivalence.
To motivate these restrictions, we first examine an important \emph{sufficient} condition for feature equivalence.

Inspired by connections to Blackwell's theory of experiments \citep{blackwell1953equivalent}, we introduce the \defword{Blackwell coarsening relation} on LC-GMs.
Returning to the general setting with potentially distinct concept spaces $\cC$ and $\cC'$, consider a \defword{transition kernel} $T \in \cK(\cC \to \cC')$.
We say $\M$ is \defword{Blackwell coarser than $\M'$ under $T$} if the following conditions hold:
\begin{enumerate}[itemsep=1pt,topsep=1pt]
    \item \textbf{Measure refinement}: $Q' = TQ$.
    \item \textbf{Kernel coarsening}: $K \sim_Q K'T$.
\end{enumerate}
For convenience, we will often simply state that $\M$ is Blackwell coarser than $\M'$, denoted $\M \Bleq \M'$, if there exists some $T \in \cK(\cC \to \cC')$ such that these relations hold.\footnote{Using $\Bleq$ is justified since Blackwell coarsening is transitive: if $\M \Bleq \M'$ under $T$ and $\M' \Bleq \M''$ under $T'$, then $\M \Bleq \M''$ under $T' T$, see \Cref{proofs:blackwell-coarsening}.}
It is straightforward to show that Blackwell coarsening is sufficient for feature equivalence:
\begin{prop}\label{prop:blackwell-duality-sufficient}
    If $\M \Bleq \M'$, then $\M$ and $\M'$ are feature equivalent.
    In particular, if $Q' = T Q$ and $K \sim_Q K' T$, then $K Q = K' Q'$.
\end{prop}
The proof follows directly from substitution and functoriality; see \Cref{proofs:blackwell-coarsening}.
Intuitively, the Blackwell coarsening relation captures a tradeoff between two explanations for randomness in a feature distribution $P^* = P^\M = P^{\M'}$.
When $\M \Bleq \M'$, $K$ is ``more stochastic" than $K'$ (a \emph{garbling} of $K$ in Blackwell's terms), whereas $Q'$ has higher entropy than $Q$.
Put differently, $\M$ emphasizes a kind of measurement uncertainty, while $\M'$ emphasizes natural variability.
This tradeoff is sharply illustrated by our example:

\begin{example}[Blackwell coarsening]
    Let $\M$ and $\M'$ be the models from \Cref{ex:feature-equivalence}.
    It is straightforward to check that $\M$ is Blackwell coarser than $\M'$ under
    \vspace{-6pt}
    \begin{equation*}
        T = \bordermatrix{
          & \mylabel{a} & \mylabel{b} \cr
        \mylabel{c} & 2/3 & 0 \cr
        \mylabel{d} & 1/3 & 1 \cr
        }.
    \end{equation*}
\end{example}

As this discussion and the example should make clear, Blackwell coarsening is not a symmetric relation, e.g., there is no Markov kernel $T' \in \cK(\cC' \to \cC)$ such that $K' = K T'$ for \Cref{ex:feature-equivalence}.
In the special case where the relation does hold in both directions (i.e., $\M \Bleq \M'$, and $\M' \Bleq \M$), we say that $\M$ and $\M'$ are \defword{Blackwell equivalent}, which we denote by $\M \Bequiv \M'$.

\paragraph{Converse}
At first pass, it may seem that Blackwell coarsening could serve as a necessary condition for feature equivalence, at least under conditions such as injectivity of $K_\sharp$ and $K'_\sharp$.
However, a simple modification of \Cref{ex:feature-equivalence}, given in \Cref{appendix:other-examples}, shows that this is not the case: two LC-GMs $\M$ and $\M'$ may be feature equivalent with neither $\M \Bleq \M'$ nor $\M' \Bleq \M$.

Despite this null result, we can already obtain a partial converse to \Cref{prop:blackwell-duality-sufficient}, which will be a key ingredient for our main results.
In particular, if the kernel coarsening condition already holds for $\M$ and $\M'$ which are feature equivalent, then the measure refinement condition automatically holds: 
\begin{lemma}\label{lemma:kernel-implies-measure}
    Assume that $K'$ is measure-separating, and that $K \sim_Q K' T$ for some $T \in \cK(\cC \to \cC')$.
    
    If $P^\M = P^{\M'}$, then $Q' = T Q$, i.e., $\M \Bleq \M'$.
\end{lemma}
The proof follows directly from injectivity of $K'$ and functoriality; see \Cref{proofs:blackwell-coarsening}.
\section{MAIN RESULTS}\label{sec:results}

We are finally ready to move on to our main results, which give \emph{necessary} conditions for the feature equivalence of two models $\M = (Q, K)$ and $\M' = (Q', K')$ in the same model class $\cM \subseteq \cM_\all(\cC, \cZ)$.
As is typical in most settings, we consider model class of the form $\cM = \cQ \times \cK$, where $\cQ \subseteq \cP(\cC)$ and $\cK \subseteq \cK(\cC \to \cZ)$, so that any constraints on the concept distributions and the mixing kernels are logically independent.

In \Cref{sec:blackwell-reducibility}, we define a natural condition on the kernel class $\cK$ under which Blackwell equivalence is indeed a necessary condition for feature equivalence.
We then use the result in \Cref{sec:intersection} to establish an easy-to-use \emph{intersection theorem} for identifiability.

\subsection{Blackwell Reducibility}\label{sec:blackwell-reducibility}

As seen in \Cref{sec:blackwell-coarsening}, a main obstacle to guaranteeing Blackwell equivalence of $\M$ and $\M'$ is the potential incomparability of the mixing kernels $K$ and $K'$.
To avoid this issue, we define a novel condition on $\cK$.

\begin{defn}\label{def:blackwell-reducibility}
    Let $\cI$ be a Polish Borel space and $\cR \defeq \MeasInj(\cC \to \cI) \circ \cP(\cC)$.
    Consider a \defword{base kernel} $B \in \cK_\sep(\cI \to \cZ)$.
    We say that $\cK \subseteq \cK(\cC \to \cZ ; \cR)$ is \defword{Blackwell reducible (BR) through $B$} if
    \begin{equation*}
        \cK \subseteq \{ B \} \circ \MeasInj(\cC \to \cI).
    \end{equation*}
    We say $\cK$ is \defword{continuous Blackwell reducible (CBR) through $B$} if 
    \begin{equation*}
        \cK \subseteq \{ B \} \circ \ContMeasInj(\cC \to \cI),
    \end{equation*}
\end{defn}
Put differently, Blackwell reducibility of $\cK$ means that, for all $K \in \cK$, we have $K = B \Dirac_{\phi}$ for some injective function $\phi\colon \cC \to \cI$.
By this definition, if $\cK$ is Blackwell reducible, then any subset $\cK' \subseteq \cK$ is also BR (and similarly for CBR).
As we show in \Cref{appendix:other-examples}, the BR condition rules out previous problematic cases: $\cK = \{ K, K' \}$ is not BR for the mixing kernels $K, K'$ in \Cref{ex:feature-equivalence}.
Assuming that $\cK$ is BR, we obtain a (specialized) converse of \Cref{prop:blackwell-duality-sufficient}, where Blackwell equivalence and feature equivalence coincide.

\begin{thm}\label{thm:feature-implies-blackwell}
    Let $\cM = \cP(\cC) \times \cK$ for some BR class of kernels $\cK$, and let $\M, \M' \in \cM$ with $\M = (Q, K)$ and $\M' = (Q', K')$.
    If $P^\M = P^{\M'}$, then $\M \Bequiv \M'$.

    If we also assume $\cK$ is a CBR class of kernels and $\supp(Q) = \supp(Q') = \cC$, then $\M$ is a Blackwell coarsening of $\M'$ under $\Dirac_\tau$ for some $\tau \in \MeasInj(\cC \to \cC)$.
\end{thm}
See \Cref{proofs:blackwell-reducibility} for the proof, which leverages the shared base kernel $B$ to reduce the problem closer to the case of deterministic mixing.
Indeed, \Cref{thm:feature-implies-blackwell} extends Lemma 2.1 of \citet{xi2023indeterminacy}, which essentially shows that feature equivalence implies Blackwell equivalence for $\cK = \{ E_g : g \in \MeasInj(\cC \to \cZ) \}$, a prototypical example of a BR class:
\begin{example}[BR for injective measurable functions]
    The set $\cK = \{ \Dirac_g : g \in \MeasInj(\cC \to \cZ) \}$ is BR.
    Specifically, letting $\cI = \cZ$ and $B\colon \bz \mapsto \dirac_\bz$, any $\Dirac_g \in \cK$ can be written as $\Dirac_g = B \Dirac_g$.
    Similarly, the set $\cK = \{ E_g \colon g \in \ContMeasInj(\cC \to \cZ \}$ is CBR.
\end{example}

However, Blackwell reducibility is a much broader notion.
For example, in latent factor models, a common assumption is that of \emph{additive noise} after mixing, i.e., $\bZ = g(\bC) + \boldeps$ for some independent noise $\boldeps \sim N_\boldeps$.
In measure-theoretic terms, this assumption imposes $K(\cdot \mid \bc) = N_\boldeps * \dirac_{g(\bc)}$, where  $*$ denotes convolution.
When $N_\boldeps$ has a well-behaved form, e.g. multivariate Gaussian, this class of mixing kernels is BR.

\begin{example}[BR for additive Gaussian noise]\label{example:br-additive-gaussian}
    Take $\cZ = \bbR^p$ and $\Sigma \in \PSD^p$, where $\PSD^p$ denotes the positive semidefinite cone.
    Let $\cK$ denote the set of mixing kernels $K$ such that $K(\cdot \mid \bc) = \cN(\bzero, \Sigma) * \dirac_{g(\bc)}$ for some $g \in \MeasInj(\cC \to \cZ)$. 
    Then $\cK$ is BR.
    
    Specifically, let $\cI = \cZ$ and $B\colon \bz \mapsto \cN(0, \Sigma) * \dirac_\bz$.
    Then, any $K \in \cK$ can be written as $K = B \Dirac_g$.
\end{example}

In \Cref{appendix:measure-separating-additive-noise}, we show that $B$ is indeed measure-separating.
While both of these examples were closely based on injective functions from $\cC$ to $\cZ$, our final example demonstrates that this need not be the case.
In particular, Blackwell reducibility can also capture parametric mixture models, as long as each mixture component is associated with a distinct distribution.

\begin{example}[BR of Gaussian mixtures]\label{example:br-gaussian-mixture}
    For $\cC = [d]$ and $\cZ = \bbR^p$, let $\cK$ denote the set of mixing kernels $K$ such that $K(\cdot \mid c) = \cN(\mu_c, \Sigma_c)$, where $\mu_c \neq \mu_{c'}$ for $c \neq c'$.
    Then $\cK$ is BR.
    
    In particular, letting $\cI = \bbR^p \times \PSD^p$ and $B\colon (\mu, \Sigma) \mapsto \cN(\mu, \Sigma)$, any mixing kernel $K \in \cK$ can be written as $K = B \Dirac_\phi$, where $\phi\colon c \mapsto (\mu_c, \Sigma_c)$.
\end{example}

As seen in \Cref{appendix:measure-separating-mixture}, the main result of \citet{yakowitz1968identifiability} can be viewed as a proof that $B$ is measure-separating on the appropriate domain.

\subsection{Intersection Theorem}\label{sec:intersection}

\Cref{thm:feature-implies-blackwell} is a powerful tool for characterizing feature equivalence, but is not directly stated in the typical language of identifiability theory.
We now bridge this gap by giving a general definition of identifiability and giving a more direct, easy-to-use version of \Cref{thm:feature-implies-blackwell} based on the intersection of two \emph{valid transition sets}, extending Theorem 2.2 of \cite{xi2023indeterminacy}.

\paragraph{Identifiability}
Let $\cM \subseteq \cM_\all(\cC, \cZ)$ be a model class, and let $\sim$ be an equivalence relation on $\cM_\all(\cC, \cZ)$.
We say that \defword{$\M \in \cM$ is identifiable up to $\sim$ in $\cM$}, denoted $\ID(\cM, \sim, \M)$, if $P^\M = P^{\M'}$ implies $\M \sim \M'$.
As this definition makes clear, identifiability is in general a \emph{ternary} relation, depending crucially on the choice model class $\cM$, the choice of equivalence relation $\sim$, and the specific model $\M \in \cM$.
However, in many practical scenarios, the same identifiability result will hold for all models $\M \in \cM$, i.e., we will have $\ID(\cM, \sim, \M)$ for all $\M \in \cM$.
In this case, we simply say that \defword{$\cM$ is identifiable up to $\sim$}, denoted $\ID(\cM, \sim)$.


Notably, Blackwell equivalence is precisely an equivalence relation on $\cM_\all(\cC, \cZ)$, i.e., under this terminology, \Cref{thm:feature-implies-blackwell} shows that $\M$ is identifiable up to Blackwell equivalence in $\cM$ whenever $\cM = \cP(\cC) \times \cK$ for some Blackwell reducible class $\cK$.
Our next goal is to characterize Blackwell equivalence more explicitly.

\paragraph{Valid transition sets}
The definition of Blackwell coarsening nearly suggests that Blackwell equivalence can be split into two separate relations: one over the concept distributions, and one over the mixing kernels.
However, the kernel coarsening condition depends on $Q$, which is somewhat unsatisfactory for the practical development of identifiability results.

To avoid this issue, we introduce a final restriction on our model classes.
Again considering a model class of the form $\cM = \cQ \times \cK$, this new restriction now applies to the set $\cQ$ of concept distributions:

\begin{defn}
    $\cQ \subseteq \cP(\cC)$ is a \defword{full support measure class} if there is a \defword{base measure} $\mu$ on $\cC$ with $\supp(\mu) = \cC$ such that for all $Q \in \cQ$, $Q \sim \mu$.
\end{defn}
Under this condition, $K \sim_Q K'$ if and only if $K \sim_\mu K'$; thus, we can discuss kernel coarsening without reference to a specific $Q \in \cQ$.
To capture the restrictions implied by the measure refinement condition, we define the set of \defword{valid measure transitions} of $\cQ$ at $Q$, denoted $\cT_Q(\cQ)$, as
\begin{equation*}\label{eqn:transition-set-Q}
    \cT_Q(\cQ)
    \defeq
    \{
    \tau \in \MeasInj(\cC \to \cC)
    :
    E_\tau Q \in \cQ
    \}.
\end{equation*}
Hence, $\cT_Q(\cQ)$ captures the set of transition functions $\tau$ which stay inside the set $\cQ$ when applied to $Q$, i.e., for which $Q'= E_\tau Q$ is still a valid concept distribution according to the model class.
Similarly, to capture the restrictions implied by the kernel coarsening condition, we define the set of \defword{valid kernel transitions} of $\cK$ at $K$, denoted $\cT_K(\cK)$, as
\begin{equation*}\label{eqn:transition-set-K}
    \cT_K(\cK)
    \defeq
    \{ \tau \in \MeasInj(\cC \to \cC) : \exists\ K' \in \cK,\ K \sim_\mu K' \Dirac_\tau \}.
\end{equation*}
Here, $\cT_K(\cK)$ denotes the set of transition functions which produce $K$ upon coarsening some allowable mixing kernel $K'$.
For ease of notation, the dependence of $\cT_K(\cK)$ on the base measure $\mu$ is left implicit.


\paragraph{Identifiability of Models}
Now that we have defined valid transition sets and introduced the measure class assumption on $\cQ$, we are ready to adapt \Cref{thm:feature-implies-blackwell} into a more usable form.
To relate the sets $\cT_Q(\cQ)$ and $\cT_K(\cK)$ to equivalence relations, we consider one of the most common forms of equivalence relations: those induced by a group of functions.
In particular, a nonempty set $G$ of bijective measurable functions from $\cC$ to $\cC$ is called a \defword{transition group} if, for all $\tau, \tau' \in G$, we have $\tau^{-1} \in G$ and $\tau \circ \tau' \in G$.

A group of functions $G$ defines a natural equivalence relation $\sim_G$ on $\cM_\all(\cC, \cZ)$, where $\M \sim_G \M'$ if and only if $\M$ is Blackwell coarser than $\M'$ under $E_\tau$ for some $\tau \in G$.
As we show in \Cref{proofs:intersection}, the group structure of $G$ ensures that $\sim_G$ is indeed an equivalence relation.
We can directly relate valid transition sets to identifiability for such relations:
\begin{thm}\label{thm:intersection}
    Let $\cM = \cQ \times \cK$, where $\cQ$ is a full support measure class and $\cK$ is CBR.
    Let $\M = (Q, K) \in \cM$, and let $G$ be a transition group.
    
    If $\cT_Q(\cQ) \cap \cT_K(\cK) \subseteq G$, then $\M$ is identifiable up to $\sim_G$ in $\cM$, i.e., $\ID(\cM, \sim_G, \M)$.
\end{thm}
\begin{proof}
    Let $\M = (Q, K)$ and $\M' = (Q', K')$ be feature equivalent models, with $\M, \M' \in \cM$.
    By the assumption that $\cQ$ is a full support measure class, $\supp(Q) = \supp(Q') = \cC$.
    Thus, by the premise that $\cK$ is CBR, \Cref{thm:feature-implies-blackwell} gives $Q' = \Dirac_\tau Q$ and $K \sim_\mu K' \Dirac_\tau$ for some $\tau \in \MeasInj(\cC \to \cC)$.
    Finally, $Q' \in \cQ$ implies $\tau \in \cT_Q(\cQ)$, and $K' \in \cK$ implies $\tau \in \cT_K(\cK)$.
    Thus, by the premise of the theorem, we have $\tau \in G$, so $\M$ is identifiable up to $\sim_G$ in $\cM$, as desired.
\end{proof}

\paragraph{Identifiability of Model Classes}

Alternatively, when we are interested in identifiability of an entire model class $\cM$, rather than identifiability of a specific model $\M = (Q, K)$, we can define transition sets which do not depend on $Q$ and $K$.
Specifically, we let $\cT(\cQ) = \bigcup_{Q \in \cQ} \cT_Q(\cQ)$ and $\cT(\cK) = \bigcup_{K \in \cK} \cT_K(\cK)$, i.e., $\tau \in \cT(\cQ)$ if it is a valid measure transition for some distribution $Q \in \cQ$, and similarly for $\tau \in \cT(\cK)$.

\begin{corollary}\label{cor:intersection-class}
    Let $\cM = \cQ \times \cK$, where $\cQ$ is a full support measure class and $\cK$ is CBR, and let $G$ be a transition group.
    If $\cT(\cQ) \cap \cT(\cK) \subseteq G$, then $\M$ is identifiable up to $\sim_G$, i.e., $\ID(\cM, \sim_G)$.
\end{corollary}

As an important case, $\cT(\cK)$ can be described quite simply for deterministic mixing functions.
In particular, let $\cK = \{ \Dirac_g : g \in \cG \}$ for some $\cG \subseteq \MeasInj(\cC \to \cZ)$.
Then, $\cT(\cK)$ is precisely the \emph{indeterminacy set} defined in \citet{xi2023indeterminacy} (see \Cref{proofs:intersection}):
\begin{equation*}
    \cT(\cK) = \{ \tau : \tau \sim_\mu g^{-1} \circ g'\ \text{for some}\ g, g' \in \cG \}.
\end{equation*}


\section{APPLICATIONS}\label{sec:applications}

In this section, we highlight the utility of \Cref{thm:intersection} by recovering classic identifiability results for two model classes relevant to unsupervised concept extraction and representation learning.
To describe these model classes, we first introduce a clean new notation that makes all modeling assumptions highly transparent.

\paragraph{Predicate notation}
To compactly express model classes of the form $\cM = \cQ \times \cK$, we introduce \defword{predicate notation}.
We encode constraints on the concept distributions $Q \in \cQ$ via \defword{concept predicates}, which are Boolean-valued functions $\Pred\colon \cP(\cC) \to \{ \True, \False \}$.
For a set $\{ \Pred_j \}_{j=1}^J$ of concept predicates, we let $\cP(\cC ; \Pred_1, \ldots, \Pred_J)$ denote the set of concept distributions for which all $J$ predicates are true:
\begin{equation*}
    \cP(\cC ; \Pred_1, \ldots, \Pred_J)
    \defeq \bigg\{ Q \in \cP(\cC) : \bigwedge_{j=1}^J \Pred_j(Q) \bigg\},
\end{equation*}
Similarly, a \defword{mixing predicate} is a Boolean-valued function $\Pred\colon \cK(\cC \to \cZ) \to \{ \True, \False \}$, and for a set of mixing predicates $\{ \Pred_j \}_{j=1}^J$, we let $\cK(\cC \to \cZ ; \Pred_1, \ldots, \Pred_J)$ denote the set of mixing kernels $K$ for which all $J$ predicates are true.

This notation plays very nicely with the notion of valid transition sets.
Specifically, let $\cQ_1 = \cP(\cC ; \Pred_1)$, $\cQ_2 = \cP(\cC ; \Pred_2)$, and $\cQ = \cP(\cC ; \Pred_1, \Pred_2)$ for some concept predicates $\Pred_1$ and $\Pred_2$.
Then a simple application of the definitions shows that $\cT(\cQ) = \cT(\cQ_1) \cap \cT(\cQ_2)$; an analogous result holds for mixing kernels.

\paragraph{Equivalence Relations}

We let $\Relabelings(d)$ denote the set of \defword{relabeling functions} on $\bbR^d$, i.e., $\tau \in \Relabelings(d)$ if and only if $\tau\colon c \mapsto \matS \cdot c$ for some permutation matrix $\matS$.
Similarly, we let $\Scale(d)$ denote the set of \defword{scaling functions} on $\bbR^d$, i.e., $\tau \in \Scale(d)$ if and only if $\tau\colon c \mapsto \matD \cdot c$ for some diagonal matrix with nonzero diagonal entries.


For $\cC = \bbR^d$, $\sim_\lbl$ denotes the equivalence relation $\sim_G$ for $G = \Relabelings(d)$, i.e., $\M \sim_\lbl \M'$ if the two models are equivalent up to relabeling the concepts.
Similarly, $\sim_\scale$ denotes the equivalence relation $\sim_G$ for $G = \{ g \circ g', g \in \Relabelings(d), g' \in \Scale(d) \}$.

\subsection{SAEs and Dictionary Learning}\label{sec:dictionary-learning}
Sparse autoencoders (SAEs) and their variants are some of the most widely-used methods for concept extraction from neural features.
Typically, sparse autoencoders assume continuous concept and feature spaces $\cC$ and $\cZ$ and high-dimensional concept spaces, i.e., $\cC = \bbR^d$ and $\cZ = \bbR^p$ for $d > p$.
As suggested by the name, SAEs aim for sparsity over the concept space, often by a soft constraint such as $\ell_1$ regularization.
Although these approaches are not typically framed from the perspective taken here, they are closely related to the classical field of \emph{dictionary learning}, also known as \emph{sparse coding} \citep{olshausen1997sparse,aharon2006k,arora2014new}.
To illustrate our results, we use \Cref{thm:intersection} to derive a classic result in dictionary learning.

To formalize a strict (rather than soft) sparsity constraint, we let $\SparseVectors^d_s$ denote the set of $s$-sparse vectors in $\bbR^d$, i.e., $\SparseVectors^d_s \defeq \{ \bc \in \bbR^d : \| \bc \|_0 \leq s \}$.
We endow $\cC = \SparseVectors^d_s$ with the \defword{canonical stratified measure} $\lambda^d_s$ defined in \Cref{appendix:canonical-stratified-measure}, noting $\supp(\lambda^d_s) = \SparseVectors^d_s$.
To employ \Cref{thm:intersection}, we use the concept predicate $\MeasureClass_\mu$, where $\MeasureClass_\mu(Q)$ holds if and only if $Q$ is in the measure class of $\mu$, i.e., $Q \ll \mu$ and $\mu \ll Q$.

In most mainstream SAEs and in the traditional dictionary learning literature, one assumes a \emph{deterministic} and \emph{linear} relationship between concepts and features, a form of the well-known \emph{linear representation hypothesis} \citep{park2024linear,jiang2024origins}.
This assumption can be encoded by the mixing predicate $\LinearMix$, where $\LinearMix(K)$ holds if and only if $K\colon \bc \mapsto \dirac_{\matG \cdot \bc}$ for some matrix $\matG \in \bbR^{p \times d}$.

Since $d > p$, the matrix $\matG$ is not necessarily injective on $\Sigma_s^d$.
To ensure injectivity, it is essential to further restrict the matrices $\matG$.
The classical approach to ensure injectivity is to assume a lower bound on the \defword{spark} of $\matG$.
In particular, we let $\spark(\matG)$ denote the smallest $k \in \bbN$ such that some $k$ columns of $\matG$ are linearly dependent.
Then, when used with the $\LinearMix$, the mixing predicate $\Spark_\sigma(K)$ holds if and only if $K\colon \bc \mapsto \delta_{\matG \cdot \bc}$ for some $\matG \in \bbR^{p \times d}$ with $\spark(\matG) \geq \sigma$.

Taking these assumptions together, we can define the model class for \emph{spark-constrained dictionary learning}:
\begin{equation}\label{eqn:m-spark}
\begin{aligned}
    \cM_\spark(d, p ; s, \sigma) 
    &\defeq \cQ_\spark \times \cK_\spark,\ \text{for}
    \\
    \cQ_\spark 
    &\defeq \cP(\Sigma^d_s ; \MeasureClass_{\lambda^d_s}),\ \text{and}
    \\
    \cK_\spark 
    &\defeq \cK(\Sigma^d_s \to \bbR^p ; \LinearMix, \Spark_\sigma).
\end{aligned}
\end{equation}

To apply \Cref{cor:intersection-class}, we first note that $\cQ_\spark$ is a full support measure class.
We also need to verify that $\cK_\spark$ is CBR.
As noted in \Cref{sec:blackwell-reducibility}, the set of continuous injective measurable functions is CBR, and any subset of a CBR class is also CBR.
Thus, to ensure the second condition, it suffices to ensure that $\cK_\spark$ contains only injective functions.

For general $\sigma$ and $s$, this condition need not be true.
However, it is well-known that taking $\sigma > 2s$ ensures injectivity.\footnote{For completeness, we give a proof in \Cref{proofs:dictionary-learning}.}
The following result further characterizes $\cT(\cK_\spark)$ when $\sigma > 2s$.

\begin{claim}\label{claim:spark}
    Let $K, K' \in \cK_\spark$ defined in \Cref{eqn:m-spark}, and assume $\sigma > 2s$.
    If $K \sim_{\lambda^d_s} K' E_\tau$ for some $\tau \in \MeasInj(\SparseVectors^d_s \to \SparseVectors^d_s)$, then  $\tau \in \Relabelings(d) \circ \Scale(d)$.

    More concisely, $\cT(\cK_\spark) = \Relabelings(d) \circ \Scale(d)$.
\end{claim}
See \Cref{proofs:dictionary-learning} for the proof of this claim: intuitively, linearity of $\tau \in \cT(\cK_\spark)$ follows from linearity of $K$ and $K'$, while the further constraints are induced by a combinatorial argument over $s$-dimensional subspaces.
Thus, by applying \Cref{claim:spark} and \Cref{cor:intersection-class}, we obtain an adaptation of Theorem 3 in \citet{aharon2006uniqueness} to the population setting:

\begin{claim}\label{claim:dictionary-learning-id}
    Let $\cM = \cM_\spark(d, p ; s, \sigma)$ for $\sigma > 2s$.
    Then $\cM$ is identifiable up to $\sim_{\textnormal{scale}}$.
\end{claim}

\subsection{Independent Component Analysis}\label{sec:ica}

Although not commonly used for concept extraction, several unsupervised representation learning approaches fit into our latent concept generative modeling framework.
These approaches are also designed with the goal of \emph{disentangling} their inputs into high-level concepts \citep{higgins2017beta,bengio2013representation,chen2018isolating,alemi2017deep}.
Again, these approaches are not always framed from a statistical angle, and often lack identifiability guarantees \citep{locatello2019challenging}.
However, many approaches are closely related to the classical statistics problem of \emph{independent component analysis (ICA)}, which fits directly into our framework.
As suggested by the name, the core assumption in ICA is that the concept variables are statistically independent.
We encode this with the concept predicate $\Ind$, where $\Ind(Q)$ holds if and only if $Q$ is a product distribution.

As in \Cref{sec:dictionary-learning}, we focus on the common setting of continuous concepts, continuous features, and deterministic mixing.
Hence, we let $\cC = \bbR^d$ and $\cZ = \bbR^p$, where $\cC$ is endowed with the Lebesgue measure $\lambda$.
In contrast to \Cref{sec:dictionary-learning}, ICA usually considers $d \leq p$, and does not assume sparsity on $\cC$.
For convenience, we also assume a unit variance constraint.
When used with $\Ind$, the concept predicate $\UnitVariance(Q)$ holds if and only if $\Cov(Q) = \matI_d$; this constraint removes scaling ambiguity and yields a simpler proof.

It is well-known that the independence constraint and deterministic mixing are not sufficient for any meaningful identifiability guarantee \citep{hyvarinen1999nonlinear}.
In our parlance, with $\cQ^d_\indQ \defeq \cP(\bbR^d ; \Ind, \UnitVariance)$ the set $\cT(\cQ^d_\indQ)$ is far too large, since several non-trivial functions preserve independence.
Hence, identifiability results in ICA must rely on additional assumptions.

Here, we focus on the classical setting of \emph{linear} ICA, which assumes a linear and injective mixing function.
We encode this assumption by the mixing predicate $\LinearInj$, and let $\cK_{\text{lin}} = \cK(\bbR^d \to \bbR^p ; \LinearInj)$.
This restriction reduces the space of possible transition kernels: since the composition of two linear maps is linear, we have $\cT(\cK_{\text{lin}}) = \GL(d)$, where $\GL(d)$ denotes the \emph{general linear group} on $\bbR^d$.

However, linear mixing still does not provide enough constraint.
In \Cref{proofs:ica}, we recount the fact that $\cT(\cK_{\text{lin}}) \cap \cT(\cQ^d_\indQ) = O(d)$, where $O(d)$ denotes the \emph{orthogonal group} on $\bbR^d$.
This ambiguity arises from a symmetry of Gaussian distributions: if $Q = \cN(\bzero, \matI_d)$, then the pushforward distribution $Q' = E_\tau Q$ also equals $\cN(\bzero, \matI_d)$ for any $\tau \in O(d)$.
To avoid this problem, the seminal work of \citet{comon1994independent} shows that this symmetry disappears when we enforce non-Gaussianity.
To state their result in our framework, we define the concept predicate $\NonGaussian$, which holds for $Q$ if and only if $Q \sim \lambda$ and at most one component of $Q$ is Gaussian.
Letting $\cQ^d_\indng \defeq \cP(\bbR^d ; \Ind, \UnitVariance, \NonGaussian)$, the next result is a minor adaptation of Theorem 11 in \citet{comon1994independent}:
\begin{claim}\label{claim:nongaussianity}
    $\cT(\cQ^d_\indng) \cap O(d) = \Relabelings(d)$.
\end{claim}
See \Cref{proofs:ica} for the proof.
Gathering these assumptions, we define the model class for linear ICA as
\begin{align*}
    \cM_\linica(d, p) \defeq \cQ^d_\indng \times \cK(\bbR^d \to \bbR^p ; \LinearInj).
\end{align*}
By construction, $\cQ^d_\indng$ is a full support measure class, and $\cK = \cK(\cC \to \cZ ; \LinearInj)$ is CBR, so we can apply \Cref{cor:intersection-class}.
Simple set manipulation shows that $\cT(\cQ^d_\indng) \cap \cT(\cK) = \Relabelings(d)$, which gives our version of Corollary 13 in \citet{comon1994independent}:
\begin{claim}
    Let $\cM = \cM_\linica(d, p)$.
    Then $\cM$ is identifiable up to $\sim_\lbl$.
\end{claim}
\section{METHODOLOGICAL IMPLICATIONS}\label{sec:methodological}
\newcommand{\sae}{\textnormal{sae}}

By design, a latent concept generative model (LC-GM) is a type of generative model over the feature space $\cZ$.
Thus, to learn an LC-GM, one may apply many of the same tools as for learning other kinds of generative models.
Here, we briefly discuss one natural workflow for algorithm development suggested by our framework, and how this workflow relates to a common model class for sparse autoencoders:
\begin{equation*}
    \cM_\sae(d, p ; s) \defeq \cP(\SparseVectors^d_s) \times \cK(\SparseVectors^d_s \to \bbR^p ; \LinearMix),
\end{equation*}
a relaxation of the model class in \Cref{eqn:m-spark}.\footnote{
Note that $\cM_\sae(d, p ; s)$ does not satisfy the same identifiability guarantees as $\cM_\spark(d, p ; s, \sigma)$.
However, one may perform post-hoc identifiability checks; see \Cref{appendix:post-hoc-checks}.
}
As shorthand, we let $\cM_\sae = \cM_\sae(d, p ; s)$.
In this section, we are slightly more informal, e.g. we assume all distributions have strictly positive densities.

\subsection{Architecture Design}

Our framework suggests a \emph{decomposed} parameterization of the model class $\cM$, similar to deep latent variable models \citep{diederik2019introduction}.

In particular, let $\Theta = \Omega \times \Psi$ for parameter spaces $\Omega$ and $\Psi$, where each $\omega \in \Omega$ yields a concept distribution $Q_\omega$ (i.e., a learnable latent ``prior"), and each $\psi \in \Psi$ yields a mixing kernel $K_\psi$ (i.e., a stochastic decoder).
Then, each $\theta = (\omega, \psi)$ yields a latent concept generative model $\M_\theta \defeq (Q_\omega, K_\psi)$, with the joint distribution $J_\theta(c,z) \defeq K_\psi(z \mid c) \cdot Q_\omega(c)$ and the induced feature distribution $P_\theta \defeq (K_\psi)_\sharp(Q_\omega)$.

One may use architectural design to exactly enforce the constraints, i.e., $Q_\omega \in \cQ$ and $K_\psi \in \cK$, or may softly enforce these constraints through regularization.
For the $\cM_\sae$ example, we require $Q_\omega$ to output $s$-sparse vectors, and we require $K_\psi$ to be linear.
The linearity constraint can be easily enforced by taking $\Psi = \bbR^{p \times d}$ and letting $K_\psi\colon c \mapsto \matG_\psi \cdot c$ be a linear layer.
Meanwhile, the sparsity constraint can be encoded into the architecture via a discrete operation, as in TopK SAEs \citep{makhzani2013k,gao2024scaling}, or can be softly enforced by $\ell_1$-regularization, as in traditional SAEs \citep{elhage2022toy}.

\subsection{Training}
After designing the architecture, one needs to train $\theta = (\omega, \psi)$ on a consistent loss function $\ell\colon \Theta \times \cZ \to \bbR$, i.e., a loss function such that the population loss
\begin{equation*}
    \cL(\theta, P^*)
    \defeq
    \bbE_{\bZ \sim P^*}[\ell(\theta, \bZ)]
\end{equation*}
is minimized if and only if $P_\theta = P^*$, assuming the model class is well-specified.\footnote{We briefly discuss misspecification in \Cref{sec:discussion}.}
In this approach, the choice of $\ell$ is constrained by the nature of $Q_\omega$ and $K_\psi$.
For example, if $Q_\omega$ provides evaluation access, and $K_\psi = E_{g_\psi}$ for some diffeomorphism $g_\psi$ with a tractable inverse $h_\psi \defeq g_\psi^{-1}$, then the negative log likelihood $\ell(\theta, z) \defeq -\log P_\theta(z)$ can be evaluated using the change of measure formula with $c = h_\psi(z)$:
\begin{equation*}
    -\log P_\theta(z) = -\log Q_\omega(c) + \log |\matJ_{g_\psi}(c)|,
\end{equation*}
where the Jacobian $\matJ_{g_\psi}(z)$ can be computed via autodifferentiation.
However, if $Q_\omega$ and $K_\psi$ are implicit models, then other choices of $\ell$, e.g. an $f$-GAN divergence \citep{nowozin2016f}, may be necessary.

\subsection{Concept Extraction as Inference}\label{sec:concept-extraction}
As described in \Cref{sec:setup}, a latent concept generative model $\M_\thetafix = (Q_\omegafix, K_\psifix)$ induces a concept extractor
\begin{equation*}
    H_\thetafix(c \mid z)
    =
    \frac{K_\psifix(z \mid c) \cdot Q_\omegafix(c)}{P_\thetafix(z)}.
\end{equation*}
Again, in the simplest setting where $K_\psi = \Dirac_{g_\psi}$ for some injective $g_\psi$, the concept extractor has a simple form: namely, $H_{\thetafix} = \Dirac_{h_\psifix}$.
For example, in $\cM_\spark$, the exact solution $c$ such that $\matG_\psifix \cdot c = z$ can be computed via orthogonal matching pursuit under additional conditions on $\matG_\psifix$ \citep{tropp2004greed}.

In the more general stochastic mixing setting, concept extraction requires \emph{posterior inference}, and the form of $Q_\omega$ and $K_\psi$ dictate the available inference procedures.
For example, if $Q_\omega$ and $K_\psi$ both provide evaluation access, then $H_\thetafix(c \mid z) \propto K_\psifix(z \mid c) \cdot Q_\omegafix(c)$ can be treated as an energy-based model, allowing for the use of inference procedures such as Hamiltonian Monte Carlo \citep{neal2011mcmc}.
However, if $Q_\omega$ and $K_\psi$ are implicit models, then one must turn to simulation-based inference \citep{cranmer2020frontier}.

If concept extraction (rather than generation) is the main goal, one may instead wish to \emph{amortize} posterior inference by learning a parameterized concept extractor, as is done for traditional SAEs.
In \Cref{appendix:amortized-posterior-inference}, we further discuss this option, and \Cref{appendix:hindupur-comparison} discusses connections between our framework and the \emph{projection encoders} framework of \citet{hindupur2025projecting}.
\section{DISCUSSION}\label{sec:discussion}

In this paper, we have introduced a unified framework which formalizes the task of concept extraction from the lens of identifiability theory.
Our two main theorems, \Cref{thm:feature-implies-blackwell} and \Cref{thm:intersection}, give powerful tools for establishing identifiability results, with our two applications in \Cref{sec:applications} showing how these theorems can be applied.
In \Cref{appendix:finite-mixture-models}, we cover a third application to a model class with stochastic mixing, which is a major extension beyond similar previous work \citep{xi2023indeterminacy}.
We expect these contributions to pave the way for a number of directions:

\textbf{Bridging the theory-practice gap.} By focusing on concept extraction as our motivating application, we have hopefully made identifiability theory accessible to a broader audience that greatly wants and needs these tools \citep{cui2025theoretical,meloux2025everything,song2025position}.
Such theory can usefully inform method development and usage, but several issues remain, e.g. reasoning about identifiability in misspecified model classes (where $P^\M \neq P^*$ for any $\M \in \cM$) and reasoning about identifiability under additional constraints on the concept extractor $H^\M$ (as brought up in \Cref{appendix:hindupur-comparison}).

\textbf{Extensions to weak supervision.} In this paper, we have considered concept extraction from the perspective of unsupervised representation learning.
However, several works on identifiability can be interpreted as providing \emph{weak supervision} on the concepts, e.g. auxiliary variables in nonlinear ICA \citep{hyvarinen2019nonlinear}, interventions in causal representation learning \citep{zhang2023identifiability,buchholz2023learning,varici2025score}, or even downstream labels in traditional representation learning \citep{roeder2021linear,reizinger2025cross,zhai2025contextures}.
Often, such supervision allows one to consider much more general latent concept generative models.
Thus, it would be very useful to extend the framework here to incorporate weak supervision.

\subsubsection*{Acknowledgements}
We would like to thank Jerry Huang, Che-Ping Tsai, and Jason Hartford for insightful discussions on the overall framework, and Yizhou Lu and Soheun Yi for careful proofreading and discussions about methodology.
We thank our anonymous reviewers for the encouragement to add \Cref{sec:methodological} and the comparisons in \Cref{appendix:hindupur-comparison}.
This research was developed with funding from the Defense Advanced Research Projects Agency (DARPA) via HR0011-25-3-0239, FA8750-23-2-1015, ONR via N00014-23-1-2368, and NSF via IIS-1909816.

\defbibheading{mybibhead}[\bibname]{
\subsubsection*{#1}
}
\printbibliography[heading=mybibhead,title={References}]

\clearpage
\onecolumn
{
\hypersetup{linkcolor=blue}
\renewcommand{\contentsname}{Contents of Appendix}
\tableofcontents
}
\addtocontents{toc}
{\protect\setcounter{tocdepth}{2}}
\newpage
\appendix
\thispagestyle{empty}


\aistatstitle{A Unifying Framework for Unsupervised Concept Extraction}

\crefalias{section}{appendix}
\crefalias{subsection}{appendix}
\crefalias{subsubsection}{appendix}

\begin{refsection}
\section{BACKGROUND}\label{appendix:background}

In \Cref{sec:notation}, we described our notation for distributions and Markov kernels.
This section reintroduces that notation and provides more explicit definitions for the unfamiliar reader.

\subsection{Borel Spaces and Distributions}
Let $\cU$ be a topological space.
The \defword{Borel space} associated with $\cU$ is the pair $(\cU, \cB)$, where $\cB$ is the \defword{Borel $\sigma$-algebra} of $\cU$, i.e., the $\sigma$-algebra generated by the open sets of $\cU$.
We define a \defword{Polish Borel space} as a Polish space (i.e., a separable completely metrizable topological space) endowed with its Borel $\sigma$-algebra.
Ignoring its topological data, such a space is commonly called a \defword{standard Borel space}, see  \citet[Definition 12.5]{kechris2012classical}.
Here, we include \emph{Polish} in the name to indicate that the topological structure is kept.

For any measurable space $\cV$, we let $\cP(\cV)$ denote the set of probability distributions on $\cV$.
We say that $Q \in \cP(\cV)$ is \defword{absolutely continuous} with respect to a measure $\mu$ on $\cV$ if $\mu(A) = 0$ implies $Q(A) = 0$ for any $\cV$-measurable set $A$.

\subsection{Markov Kernels}
A \defword{Markov kernel} from $\cU$ to $\cV$ is a measurable function from $\cU$ to $\cP(\cV)$.
We denote the set of Markov kernels from $\cU$ to $\cV$ as $\cK(\cU \to \cV)$, and given a Markov kernel $K \in \cK(\cU \to \cV)$, we let $K(\cdot \mid u)$ denote its output for $u \in \cU$.
Given a measure $\mu$ on $\cU$, we say that two kernels $K$ and $K'$ are \defword{equal $\mu$-almost everywhere} if $\mu(N) = 0$ for the set $N = \{ u \in \cU : K(\cdot \mid u) \neq K'(\cdot \mid u) \}$.
We write $K \sim_\mu K'$ to denote that $K$ and $K'$ are equal $\mu$-almost everywhere.

\paragraph{Pushforwards and Posterior Kernels}
For a Markov kernel $K \in \cK(\cU \to \cV)$, we define an associated operator $K_\sharp\colon \cP(\cU) \to \cP(\cV)$, which maps the distribution $Q \in \cP(\cU)$ to the distribution $P = K_\sharp(Q)$ such that
\begin{equation}\label{eqn:kernel-composition}
    P(\cdot) = \int_\cU K(\cdot \mid u) \cdot Q(du).
\end{equation}

Given a Markov kernel $K \in \cK(\cU \to \cV)$ and a distribution $Q \in \cP(\cU)$, a \defword{posterior kernel} of $K$ with respect to $Q$ is a Markov kernel $L \in \cK(\cV \to \cU)$ such that, for all measurable sets $A \subseteq \cU$ and $B \subseteq \cV$,
\begin{equation}\label{eqn:bayesian-inverse}
    \int_A K(B \mid du) \cdot Q(du)
    =
    \int_B L(A \mid dv) \cdot P(dv),
\end{equation}
where $P = K_\sharp(Q)$.
Assuming that $\cU$ and $\cV$ are standard Borel spaces, a posterior kernel always exists, and is unique $P$-almost everywhere (for reference, see Example 11.6 in \citet{fritz2020synthetic}).
We let $K^\dagger_Q$ denote any such posterior kernel.

\paragraph{Composition}
Given two Markov kernels $K \in \cK(\cU \to \cV)$ and $L \in \cK(\cV \to \cW)$, the \defword{composition} of these Markov kernels, denoted $LK$, is a Markov kernel from $\cU$ to $\cW$, where for all $u \in \cU$,
\begin{equation*}
    (LK)(\cdot \mid u) = \int_\cV L(\cdot \mid v) \cdot K(dv \mid u)
\end{equation*}
A basic but important fact is that the pushforward and composition operators can be interchanged, i.e., $(LK)_\sharp = L_\sharp \circ K_\sharp$.
This property is actually a special case of the associativity of Markov kernels, since $Q \in \cP(\cV)$ can be viewed as a Markov kernel from the empty set to $\cV$, in which case $K_\sharp(Q) = KQ$.
For a general proof of associativity, see Proposition 3.2 of \citet{panangaden1999category}.

\paragraph{Composition Notation}
As defined in \Cref{sec:notation}, for two sets of Markov kernels $\cK \subseteq \cK(\cU \to \cV)$ and $\cL \subseteq \cK(\cV \to \cW)$, we let
\begin{equation*}
    \cL \circ \cK \defeq \{ LK : L \in \cL, K \in \cK \} \subseteq \cK(\cU \to \cW).
\end{equation*}
For a set of probability distributions $\cP \subseteq \cP(\cU)$, we let 
\begin{equation*}
    \cK \circ \cP \defeq \{ KP : K \in \cK, P \in \cP \} \subseteq \cP(\cV).
\end{equation*}
For a set of functions $\cF \subseteq \cF(\cU \to \cV)$, we let 
\begin{equation*}
    \cL \circ \cF \defeq \{ LE_f : L \in \cL, f \in \cF \}\ \subseteq \cK(\cU \to \cW).
\end{equation*}
and
\begin{equation*}
    \cF \circ \cP \defeq \{ E_f P : f \in \cF, P \in \cP \} \subseteq \cP(\cV).
\end{equation*}

\subsection{Technical Results for Composition}
\newcommand{\Equalizer}{\textnormal{Equalizer}}

Normal composition (on the left) respects almost-everywhere equality, i.e., if $K \sim_\mu K'$, then $LK \sim_\mu LK'$.
In several proofs, we need to \emph{pre-compose} two kernels which are equal almost everywhere with respect to some measure.
Thus, we will show that pre-composition also respects a.e.-equality in the necessary sense.
First, we prove two auxiliary claims:
\begin{claim}\label{claim:equalizer-set-measurable}
    Let $\cU$ and $\cV$ be standard Borel spaces.
    For any measurable functions $f, g\colon \cU \to \cV$, the equalizer set
    \begin{equation*}
        \Equalizer(f, g) \defeq \{ u \in \cU : f(u) = g(u) \}
    \end{equation*}
    is measurable.
\end{claim}
\begin{proof}
    Defining $h\colon \cU \to \cV \times \cV$ with $h\colon u \mapsto (f(u), g(u))$, we have that $\Equalizer(f, g) = h^{-1}[\Delta_\cV]$, where $\Delta_\cV = \{ (v, v) : v \in \cV \}$.
    As noted by \cite{kallenberg2011invariant}, the diagonal $\Delta_\cV$ is measurable in $\cV \times \cV$ when $\cV$ is Borel, and $h$ is measurable by construction.
    Thus, $h^{-1}[\Delta_\cV]$ is a measurable set.
\end{proof}

\begin{claim}\label{claim:precomposition-full-measure}
    Take $R \in \cP(\cU)$ and $J \in \cK(\cU \to \cV)$, and let $R' = JR \in \cP(\cV)$.
    Let $\cV_1 \subseteq \cV$ be a measurable set with $R'(\cV_1) = 1$.
    Define
    \begin{equation*}
        \cU_1 = \{ u \in \cU : J(\cV_1 \mid u) = 1 \}.
    \end{equation*}
    Then $\cU_1$ is measurable and $R(\cU_1) = 1$.
\end{claim}
\begin{proof}
    Define the measurable function $f\colon \cU \to \bbR$ with $f(u) \colon u \mapsto J(\cV_1 \mid u)$, and let $g\colon \cU \to \bbR$ with $g(u) = 1$.
    Then $\cU_1$ is the equalizer set of $f$ and $g$, so it is measurable by \Cref{claim:equalizer-set-measurable}.
    Then, by definition of the pushforward,
    \begin{equation*}
        R'(\cV_1) 
        =
        (JR)(\cV_1)
        =
        \int_\cU f(u) \cdot R(du)
        =
        1.
    \end{equation*}
    By linearity of the Lebesgue integral, we have $\int_\cU (1 - f(u)) \cdot R(du) = 0$.
    Since $1 - f(u) \geq 0$, we must have $1 - f(u) = 0$ almost surely on $R$. Thus, $f(u) = J(\cV_1 \mid u) = 1$ almost surely on $R$; i.e., $R(\cU_1) = 1$.
\end{proof}

\begin{claim}[Left composition respects a.e.-equality]\label{claim:composition-almost-everywhere}
    Take $R \in \cP(\cU)$, $J \in \cK(\cU \to \cV)$, and $L_1, L_2 \in \cK(\cV \to \cW)$.
    
    Let $R' = JR \in \cP(\cV)$. If $L_1 \sim_{R'} L_2$, then $L_1 J \sim_R L_2 J$.
\end{claim}
\begin{proof}
    \newcommand{\eq}{\text{eq}}
    Define the sets
    \begin{align*}
        \cV_\eq &\defeq \{ v \in \cV : L_1(\cdot \mid v) = L_2(\cdot \mid v) \},
        \\
        \cU_\eq &\defeq \{ u \in \cU : (L_1 J)(\cdot \mid u) = (L_2 J)(\cdot \mid u) \},\ \text{and}
        \\
        \cU_1 &\defeq \{ u \in \cU : J(\cV_\eq \mid u) = 1 \}.
    \end{align*}
    We briefly check these sets are measurable:
    \begin{itemize}
        \item $\cV_\eq$ is the equalizer set of the measurable functions $f, g\colon \cV \to \cP(\cW)$ with $f\colon v \mapsto L_1(\cdot \mid v)$ and $g\colon v \mapsto L_1(\cdot \mid v)$, and is thus measurable by \Cref{claim:equalizer-set-measurable}.
        \item $\cU_\eq$ is the equalizer set of the measurable functions $f, g\colon \cU \to \cP(\cW)$ with $f\colon u \mapsto (L_1 J)(\cdot \mid u)$ and $g\colon u \mapsto (L_2 J)(\cdot \mid u)$, and is thus measurable by \Cref{claim:equalizer-set-measurable}.
        \item $\cU_1$ is the equalizer set of the measurable functions $f, g\colon \cU \to \bbR$ with $f\colon u \mapsto J(\cV_\eq \mid u)$ and $g\colon u \mapsto 1$, and is thus measurable by \Cref{claim:equalizer-set-measurable}.
    \end{itemize}

    Notably, we use that $\cP(\cW)$ is a Polish Borel space, which follows from the fact that $\cP(\cW)$ is a Polish space (under the topology of weak convergence) for any Polish space $\cW$, see e.g. \citet[Theorem 8.9.5]{bogachev2007measure}.
    
    By the assumption that $L_1 \sim_{R'} L_2$, we have $R'(\cV_\eq) = 1$.
    By \Cref{claim:precomposition-full-measure}, we have $R(\cU_1) = 1$.

    Finally, we aim to show that $\cU_1 \subseteq \cU_\eq$.
    For $u \in \cU_1$, we have
    \begin{align*}
        (L_1 J)(\cdot \mid u)
        &=
        \int_\cV L_1(\cdot \mid v) \cdot J(dv \mid u)
        \tag{By definition}
        \\
        &=
        \int_{\cV_\eq} L_1(\cdot \mid v) \cdot J(dv \mid u)
        \tag{Since $J(\cV_1 \mid u) = 1$ for $u \in \cU_1$}
        \\
        &=
        \int_{\cV_\eq} L_2(\cdot \mid v) \cdot J(dv \mid u)
        \tag{Since $L_1(\cdot \mid v) = L_2(\cdot \mid v)$ for $v \in \cV_\eq$}
        \\
        &=
        \int_{\cV} L_2(\cdot \mid v) \cdot J(dv \mid u)
        \tag{Since $J(\cV_1 \mid u) = 1$ for $u \in \cU_1$}
        \\
        &=
        (L_2 J)(\cdot \mid u)
        \tag{By definition}
    \end{align*}
    Thus $\cU_1 \subseteq \cU_\eq$, hence $R(\cU_\eq) \geq R(\cU_1) = 1$.
    Hence, by the definition of $\cU_\eq$, we have $L_1 J \sim_R L_2 J$.
\end{proof}

\section{ADDITIONAL EXAMPLES}\label{appendix:other-examples}

\subsection{Feature Equivalence Does Not Imply Blackwell Comparability}

In \Cref{sec:blackwell-coarsening}, we mentioned that a minor modification of \Cref{ex:feature-equivalence} shows that two models $\M$ and $\M'$ may be feature equivalent, but may not be comparable under the Blackwell relation, i.e., we may have neither $\M \Bleq \M'$ nor $\M' \Bleq \M$.
We now furnish the promised example.

\begin{example}[Feature equivalence]
    Consider the spaces $\cC = \{ a, b \}$, $\cC' = \{ c, d \}$, and $\cZ = \{ u, v \}$.
    
    Let $\M = (Q, K)$ for
    \begin{equation*}
        Q = \bordermatrix{
          & \cr
        a & 3/4 \cr
        b & 1/4
        }
        \quad
        K = \bordermatrix{
          & a & b \cr
        u & 2/3 & 0 \cr
        v & 1/3 & 1 \cr
        }.
    \end{equation*}

    Meanwhile, let $\M' = (Q', K')$ for
    \begin{equation*}
        Q' = \bordermatrix{
        & \cr
        c & 1/2 \cr
        d & 1/2
        }
        \quad
        K' = \bordermatrix{
          & c & d \cr
        u & 3/4 & 1/4 \cr
        v & 1/4 & 3/4 \cr
        }.
    \end{equation*}
    Then $P^\M = P^{\M'} = \frac{1}{2}(\delta_u + \delta_v)$.
    We do not have $K \sim_Q K' T$ for any Markov kernel $T \in \cK(\cC \to \cC')$.
    In particular, since $Q$ has full support, this condition becomes an equality, i.e., $K = K'T$.
    This equality induces a well-posed system of linear equations:
    \begin{equation}
        \begin{bmatrix}
            2/3 & 0
            \\
            1/3 & 1
        \end{bmatrix}
        =
        \begin{bmatrix}
            T_{ac} & T_{bc}
            \\
            T_{ad} & T_{bd}
        \end{bmatrix}
        \cdot 
        \begin{bmatrix}
            3/4 & 1/4
            \\
            1/4 & 3/4
        \end{bmatrix}.
    \end{equation}
    The only solution to this system is
    \begin{equation*}
        \begin{bmatrix}
            T_{ac} & T_{bc}
            \\
            T_{ad} & T_{bd}
        \end{bmatrix}
        =
        \begin{bmatrix}
            5/6 & -1/2
            \\
            1/6 & 3/2
        \end{bmatrix},
    \end{equation*}
    which is not a valid Markov kernel (since, e.g., $T_{bc} < 0$).
    
    Similarly, we do not have $K' \sim_{Q'} KT'$ for any Markov kernel $T' \in \cK(\cC' \to \cC)$.
    Setting up a similar system of linear equations, the only solution is
    \begin{equation*}
        \begin{bmatrix}
            T'_{ca} & T'_{da}
            \\
            T'_{cb} & T'_{db}
        \end{bmatrix}
        =
        \begin{bmatrix}
            9/8 & 3/8
            \\
            -1/8 & 5/8
        \end{bmatrix},
    \end{equation*}
    which is again not a valid Markov kernel.
\end{example}

\subsection{A Blackwell Irreducible Class of Kernels}

In \Cref{sec:blackwell-reducibility}, we mentioned that the Blackwell reducibility condition rules out cases like \Cref{ex:feature-equivalence}.
We now describe this statement in more detail.

To adapt \Cref{ex:feature-equivalence} to the case where $\cC = \cC'$, we let $\cC = \{ a, b\}$ and $\cZ = \{ u, v \}$.
We consider the kernel class $\cK = \{ K, K' \} \subseteq \cK_\sep(\cC \to \cZ)$, where
\begin{equation}
    K = \bordermatrix{
      & a & b \cr
    u & 2/3 & 0 \cr
    v & 1/3 & 1 \cr
    }\quad \text{and}\quad
    K' = \bordermatrix{
      & a & b \cr
    u & 1 & 0 \cr
    v & 0 & 1 \cr
    }.
\end{equation}

Then, we can show the following.
\begin{claim}
    The above kernel class $\cK$ is not Blackwell reducible.
\end{claim}
\begin{proof}
Suppose there exists a Polish Borel space $\cI$, a base kernel $B \in \cK(\cI \to \cZ)$, and functions $\phi, \phi' \in \MeasInj(\cC \to \cZ)$ such that $K = B E_\phi$ and $K' = B E_{\phi'}$.
We will show that the base kernel $B$ cannot be measure-separating.

Let $R = E_\phi Q$ and $R' = E_{\phi'} Q'$ for
\begin{equation}
    Q = \bordermatrix{
        & \cr
        a & 3/4 \cr
        b & 1/4
    }\quad \text{and}\quad
    Q' = \bordermatrix{
        & \cr
        a & 1/2 \cr
        b & 1/2
    }.
\end{equation}
Then we have $R = \frac{3}{4}\delta_{\phi(a)} + \frac{1}{4} \delta_{\phi(b)}$ and $R' = \frac{1}{2} \delta_{\phi'(a)} + \frac{1}{2} \delta_{\phi'(b)}$.
Since $\phi$ and $\phi'$ are injective, we must have $R \neq R'$, since (for example), $R$ has lower entropy than $R'$.
However, we have $B R = K Q = \frac{1}{2} (\delta_u + \delta_v)$ and $B R' = K' Q' = \frac{1}{2} (\delta_u + \delta_v)$, i.e., $B R = BR'$ for $R \neq R'$, so $B$ is not measure-separating.
\end{proof}
\section{DEFEFFED PROOFS (MAIN RESULTS)}

\subsection[Proofs from Section 3.3]{Proofs from \Cref{sec:blackwell-coarsening}}\label{proofs:blackwell-coarsening}

As mentioned in \Cref{sec:blackwell-coarsening}, the Blackwell coarsening relation is transitive, justifying the use of the ordering symbol $\Bleq$.
We now state this formally and give a short proof.
\begin{prop}[Transitivity of the Blackwell coarsening relation]\label{prop:transitivity-blackwell}
    Let $\M = (Q, K)$, $\M' = (Q', K')$, and $\M'' = (Q'', K'')$, where $\M \in \cM_\all(\cC, \cZ)$, $\M' \in \cM_\all(\cC', \cZ)$, and $\M'' \in \cM_\all(\cC'', \cZ)$.
    
    If $\M \Bleq \M'$ and $\M' \Bleq \M''$, then $\M \Bleq \M''$. In particular, if $\M$ is Blackwell coarser than $\M'$ under $T$, and $\M'$ is Blackwell coarser than $\M''$ under $T'$, then $\M$ is Blackwell coarser than $\M''$ under $T'' = T' T$.
\end{prop}
\begin{proof}
    From the Blackwell coarsening relation, we have (1) $Q' = TQ$, (2) $K \sim_Q K' T$, (3) $Q'' = T' Q'$, and (4) $K' \sim_{Q'} K'' T'$.
    We must check the measure refinement and kernel coarsening conditions.

    \textbf{Measure refinement:}
    The proof follows from substitution and functoriality:
    \begin{align*}
        Q'' &= T' Q'
        \tag{Fact (3)}
        \\
        &= T'(TQ)
        \tag{Fact (1)}
        \\
        &= (T' T)(Q)
        \tag{Functoriality}
        \\
        &= T''Q
        \tag{Since $T'' \defeq T' T$},
    \end{align*}
    as desired.

    \textbf{Kernel coarsening:}
    By \Cref{claim:composition-almost-everywhere}, we have that (4) $K' \sim_{Q'} K'' T'$ implies that (5) $K' T \sim_Q (K''T') T$.\footnote{
    In particular, we use \Cref{claim:composition-almost-everywhere} with $R = Q$, $J = T$, $L_1 = K'$, $L_2 = K''T'$, and $R' = Q' = TQ$.
    }
    The proof follows from transitivity of $\sim_Q$ and functoriality:
    \begin{align*}
        K &\sim_Q K'T
        \tag{Fact (1)}
        \\
        &\sim_Q (K''T')T
        \tag{Fact (5) and transitivity of $\sim_Q$}
        \\
        &= K'' (T' T)
        \tag{By functoriality}
        \\
        &= K''T'',
        \tag{Since $T'' \defeq T' T$}
    \end{align*}
    as desired.
\end{proof}

Our first result in \Cref{sec:blackwell-coarsening} shows that Blackwell coarsening is sufficient for feature equivalence.
We restate this result here and provide the proof.

\begin{prop*}[Restatement of \Cref{prop:blackwell-duality-sufficient}]
    If $\M \Bleq \M'$, then $\M$ and $\M'$ are feature equivalent.
    In particular, if $Q' = T Q$ and $K \sim_Q K' T$, then $K Q = K' Q'$.
\end{prop*}
\begin{proof}
    Since $K \sim_Q K'T$, we have $K Q = (K'T)Q$.
    Since $Q' = TQ$, we have $K'Q' = K'(TQ)$.
    
    Thus, $P^\M = P^{\M'}$ directly follows from functoriality, since $(K'T)Q = K'(TQ)$.
\end{proof}

Our second result in \Cref{sec:blackwell-coarsening} gives a result in the reverse direction: under an injectivity assumption, feature equivalence and the kernel coarsening condition imply the measure refinement condition.
Again, we restate this result and provide the proof.
\begin{lemma*}[Restatement of \Cref{lemma:kernel-implies-measure}]
    Assume that $K'$ is measure-separating, and that $K \sim_Q K' T$ for some $T \in \cK(\cC \to \cC')$.
    
    If $P^\M = P^{\M'}$, then $Q' = T Q$, i.e., $\M \Bleq \M'$.
\end{lemma*}
\begin{proof}
    By substitution and functoriality,
    \begin{align*}
        K Q
        &=
        (K'T)(Q)
        \tag{Since $K \sim_Q K'T$}
        \\
        &=
        K' (T Q)
        \tag{Functoriality}.
    \end{align*}
    
    By assumption, $K Q = K'Q'$, hence the above gives that $K'(TQ) = K'(Q')$.
    By the assumption that $K'$ is measure-separating, $Q' = TQ$.
\end{proof}
\subsection[Proofs from Section 4.1]{Proofs from \Cref{sec:blackwell-reducibility}}\label{proofs:blackwell-reducibility}

Our first result in \Cref{sec:blackwell-reducibility} shows that feature equivalance and Blackwell equivalence coincide when we only consider mixing kernels from a Blackwell reducible class $\cK$.
We now prove this key result.
\begin{thm*}[Restatement of \Cref{thm:feature-implies-blackwell}]
    Let $\cM = \cP(\cC) \times \cK$ for some BR class of kernels $\cK$, and let $\M, \M' \in \cM$ with $\M = (Q, K)$ and $\M' = (Q', K')$.
    If $P^\M = P^{\M'}$, then $\M \Bequiv \M'$.

    If we also assume $\cK$ is a CBR class of kernels and $\supp(Q) = \supp(Q') = \cC$, then $\M$ is a Blackwell coarsening of $\M'$ under $\Dirac_\tau$ for some $\tau \in \MeasInj(\cC \to \cC)$.
\end{thm*}
\begin{proof}
    By Blackwell reducibility, $K = B \Dirac_\phi$ and $K' = B \Dirac_{\phi'}$ for some base kernel $B \in \cK_\sep(\cI \to \cZ ; \cR)$ and injective functions $\phi$ and $\phi'$.
    Thus, by functoriality, $P^\M = B(E_\phi Q)$ and $P^{\M'} = B(E_{\phi'}Q')$.
    By definition, $\cR = \MeasInj(\cC \to \cI) \circ \cP(\cC)$, so $E_\phi, E_{\phi'} \in \cR$.
    Since $B$ is measure-separating on $\cR$, and $P^\M = P^{\M'}$, we have $E_\phi Q = E_{\phi'} Q' = R$ for some $R \in \cP(\cI)$.

    To address domain issues, we must define the restrictions of several objects:
    \begin{itemize}
        \item \textbf{Sets:} Let $\cI_0 = \supp(R)$, let $\cC_0 = \phi^{-1}[\cI_0]$, and let $\cC'_0 = (\phi')^{-1}[\cI_0]$.
        \item \textbf{Functions:} Let $\phi_0\colon \cC_0 \to \cI_0$ denote the restriction of $\phi$ to $\cC_0$, and let $\phi_0'\colon \cC'_0 \to \cI_0$ denote the restriction of $\phi'$ to $\cC'_0$.
        \item \textbf{Kernels:} Let $K_0 \in \cK(\cC_0 \to \cZ)$ denote the restriction of $K$ to $\cC_0$, and let $K'_0 \in \cK(\cC'_0 \to \cZ)$ denote the restriction of $K'$ to $\cC'_0$.
        Note that $K_0 = B E_{\phi_0}$ and $K'_0 = B E_{\phi'_0}$.
    \end{itemize}

    By construction, $\phi_0$ and $\phi_0'$ have matching images, so $\tau_0 \defeq (\phi'_0)^{-1} \circ \phi_0$ is a well-defined function from $\cC_0$ to $\cC'_0$, and $\tau_0$ is injective.
    Let $\tau$ be any measurable extension of $\tau_0$ onto all of $\cC$.
    We now show that $\M$ is Blackwell coarser than $\M'$ under $E_\tau$.


    \paragraph{Kernel Coarsening}
    By definition of $\tau_0$, we have $E_{\phi'_0} E_{\tau_0} = E_{\phi_0}$.
    Thus, by functoriality, $K'_0 E_{\tau_0} = (B E_{\phi'_0}) E_{\tau_0} = B(E_{\phi'_0} E_{\tau_0}) = B E_{\phi_0} = K_0$, i.e., $K_0 = K'_0 E_{\tau_0}$.
    Since $\tau$ only extends $\tau_0$ on a $Q$-null set, this gives $K \sim_Q K' \Dirac_\tau$.

    \paragraph{Measure refinement}
    By Blackwell reducibility, $K'$ is measure-separating, and we have just shown that $K \sim_Q K' \Dirac_\tau$.
    Thus, by \Cref{lemma:kernel-implies-measure}, we have $Q' = E_\tau Q$.

    Repeating the proof by swapping $\M$ and $\M'$ shows that $\M'$ is Blackwell coarser than $\M$ under any extension $\tau'$ of $\tau'_0 \defeq \phi_0^{-1} \circ \phi_0'$.
    Hence, $\M$ and $\M'$ are Blackwell equivalent.

    \paragraph{Special case of full supports and continuous Blackwell reducibility}
    By \Cref{claim:precomposition-full-measure}, since $R(\cI_0) = 1$, we have that $Q(\cC_0) = 1$.
    By definition, $\cI_0$ is a closed set.
    Thus, if $\phi$ is continuous, then $\cC_0 = \phi^{-1}[\cI_0]$ is a closed set.
    Hence, by definition of $\supp(Q)$, we have $\supp(Q) \subseteq \cC_0$.
    Similarly, $\supp(Q') \subseteq \cC'_0$.

    Under the assumption that $\supp(Q) = \supp(Q') = \cC$, we thus obtain $\cC_0 = \cC'_0 = \cC$.
    Hence, we avoid all domain issues: $\tau_0$ is already a well-defined injective function from $\cC$ to $\cC$.
\end{proof}

\subsubsection[Measure-separating base kernel under additive noise]{Measure-separating base kernel under additive noise}\label{appendix:measure-separating-additive-noise}

To prove that the base kernel in \Cref{example:br-additive-gaussian} is measure-separating, we use elementary properties of characteristic functions, which we now review; see \citet{lukacs1970characteristic} for reference.
Given a probability distribution $R$ over $\cU = \bbR^d$, we recall that the \defword{characteristic function} of $R$ is defined as $\phi_R(t) \defeq \bbE_{\bU \sim R}[e^{it^\top \bU}]$, and that the map $R \mapsto \phi_R$ is injective, i.e., characteristic functions are unique.
Further, given two distributions $R$ and $N$ over $\cU$, we have $\phi_{N * R}(t) = \phi_N(t) \cdot \phi_R(t)$.

\begin{claim}\label{claim:convolution-measure-separating}
    Let $N_\boldeps \in \cP(\cZ)$ such that $\phi_{N_\boldeps}$ is nowhere zero.
    Then the base kernel $B\colon z \mapsto N_\boldeps * \delta_z$ is measure-separating.
\end{claim}
\begin{proof}
    Consider $R, R' \in \cP(\cZ)$.
    Let $Q = BR = N_\boldeps * R$ and $Q' = BR' = N_\boldeps * R'$.
    Then we have $\phi_Q(t) = \phi_{N_\boldeps}(t) \cdot \phi_R(t)$ and $\phi_{Q'}(t) = \phi_{N_\boldeps}(t) \cdot \phi_{R'}(t)$.

    Suppose $Q = Q'$.
    Then $\phi_Q = \phi_{Q'}$, and since $\phi_{N_\eps}$ is nowhere zero, we must have $\phi_R = \phi_{R'}$.
    Since characteristic functions are unique, this entails $R = R'$, i.e., $B$ is measure-separating, as desired.
\end{proof}

The fact that the base kernel in \Cref{example:br-additive-gaussian} is measure-separating follows from \Cref{claim:convolution-measure-separating}, since the characteristic function of $N_\boldeps = \cN(0, \Sigma)$ is $\phi_{N_\boldeps}(t) = \exp(-t^\top \Sigma t/2)$, which is nowhere zero.

\subsubsection[Measure-separating base kernel under additive noise]{Measure-separating base kernel in Gaussian mixture}\label{appendix:measure-separating-mixture}

Recall that, in \Cref{example:br-gaussian-mixture}, we have the concept $\cC = [d]$, the index space $\cI = \bbR^p \times \PSD^p$, and the feature space $\cZ = \bbR^p$.
We defined the base kernel $B: (\mu, \Sigma) \mapsto \cN(\mu, \Sigma)$, and we wished to show that $B$ is measure-separating over measures in $\cR \defeq \{ E_\phi Q : \phi \in \MeasInj(\cC \to \cI), Q \in \cP(\cC) \}$.
In particular, $\cR$ contains all distributions of the form $R = \sum_{i=1}^r w_i \delta_{\mu_i,\Sigma_i}$, where $r \leq d$ and $(\mu_i, \Sigma_i) \neq (\mu_j, \Sigma_j)$ for $i \neq j$.

Consider $R, R' \in \cR$ with $R = \sum_{i=1}^r w_i \delta_{\mu_i,\Sigma_i}$ and $R' = \sum_{i=1}^{r'} w'_i \delta_{\mu'_i,\Sigma'_i}$.
Suppose $Q = B R$ and $Q' = B R'$ with $Q = Q'$.
Then, by linearity of the pushforward operator, $B_\sharp(R - R') = 0$, where $R - R' \in \cR - \cR$ is a \defword{signed measure} on $\cI$, and $\cR - \cR \defeq \{ R - R' : R, R' \in \cR \}$.
Thus, if $\operatorname{nullspace}(B_\sharp|_{\cR - \cR}) = \{ 0 \}$, we can conclude that $R - R' = 0$, i.e., $R = R'$.
This connection is essentially the same one described in the first theorem of \citet{yakowitz1968identifiability}, where $B$ is treated as an indexed family of distributions (rather than as a Markov kernel), and the result is stated in terms of linear independence (rather than in terms of the nullspace).
In particular, Proposition 2 of \citet{yakowitz1968identifiability} can be viewed as showing the more general result that $\operatorname{nullspace}(B_\sharp|_{\cS}) = \{ 0 \}$, where $\cS$ is the set of signed measures with finite (discrete) support (this result is more general since $\cR - \cR \subset \cS$).


\subsection[Proofs from Section 4.2]{Proofs from \Cref{sec:intersection}}\label{proofs:intersection}

In \Cref{sec:intersection}, we mentioned that any group $G$ of functions from $\cC$ to $\cC$ defines an equivalence relation $\sim_G$ over $\cM_\all(\cC, \cZ)$.
We now formally state and prove this result.
\begin{prop}
    Let $G$ be a group of functions on $\cC$, and define the binary relation $\Rel_G$ on $\cM_\all(\cC, \cZ)$ as follows: $\Rel_G(\M, \M')$ holds if and only $\M$ is Blackwell coarser than $\M'$ under some $\tau \in G$.

    Then, $\Rel_G$ is an equivalence relation.
\end{prop}
\begin{proof}
    We check the three axioms of an equivalence relation:
    
    \textbf{Reflexivity:} Since $G$ is a group, it contains the identity map $\text{id}_\cC\colon c \mapsto c$. 
    Clearly, $\M$ is a Blackwell coarsening of itself under the identity map, so $\Rel_G(\M, \M)$ holds.

    \textbf{Symmetry:} Let $\M = (Q, K)$ be Blackwell coarser than $\M' = (Q', K')$ under $\tau \in G$, i.e., $Q' = \Dirac_\tau Q$ and $K \sim_Q K'\Dirac_\tau$.
    Since $G$ is a group, we have $\tau^{-1} \in G$.

    We now show that $\M'$ is Blackwell coarser than $\M$ under $\tau^{-1}$.
    \begin{itemize}
        \item \emph{Measure refinement:} By substitution and functoriality, we have $E_{\tau^{-1}}(Q') = (E_{\tau^{-1}} E_\tau)(Q)$.
        Since $\tau$ is injective, this gives $E_{\tau^{-1}} Q' = Q$, as desired.
        \item \emph{Kernel coarsening:} 
        By \Cref{claim:composition-almost-everywhere}, $K \sim_Q K' \Dirac_\tau$ implies that $K \Dirac_{\tau^{-1}} \sim_{Q'} (K' \Dirac_\tau) \Dirac_{\tau^{-1}}$.\footnote{
        In particular, we use \Cref{claim:composition-almost-everywhere} with $R = Q'$, $J = E_{\tau^{-1}}$, $L_1 = K$, $L_2 = K' E_{\tau}$, and $R' = Q = E_{\tau^{-1}} Q'$.
        }
        By functoriality and injectivity of $\tau$, we obtain $K E_{\tau^{-1}} \sim_{Q'} K'$, as desired.
    \end{itemize}

    \textbf{Transitivity:} Let $\M$ be a Blackwell coarsening of $\M'$ under $E_\tau$ and let $\M'$ be a Blackwell coarsening of $\M''$ under $E_{\tau'}$.
    By transitivity of Blackwell coarsening (see \Cref{prop:transitivity-blackwell}), $\M$ is a Blackwell coarsening of $\M''$ under $E_{\tau''} = E_\tau E_{\tau'}$.
    Since $G$ is a group of functions, $\tau'' \in G$, hence $\Rel_G(\M, \M'')$ holds.
\end{proof}

At the end of \Cref{sec:intersection}, we mentioned that the valid transition set $\cT(\cK)$ takes a special form when $\cK$ contains only injective deterministic functions from some set $\cG$.
As in \citet{xi2023indeterminacy}, we need to assume that $\cG$ contains only \emph{image equivalent functions}, i.e., $g[\cC] = g'[\cC]$ for any $g, g' \in \cG$, so that $g^{-1} \circ g'$ is well-defined.
We now prove this claim.
\begin{prop}[Valid kernel transitions in the deterministic case]
    Let $\cG \subseteq \MeasInj(\cC \to \cZ)$ contain image equivalent functions, and assume $\cK = \{ \Dirac_g : g \in \cG \}$.
    Then $\cT(\cK) = \{ \tau : \tau \sim_\mu g^{-1} \circ g'\ \text{for some}\ g, g' \in \cG \}$.
\end{prop}
\begin{proof}
    Let $\cT(\cG) = \{ \tau : \tau \sim_\mu g^{-1} \circ g'\ \text{for some}\ g, g' \in \cG \}$.
    We first prove the $\cT(\cK) \subseteq \cT(\cG)$, then that $\cT(\cG) \subseteq \cT(\cK)$.

    \textbf{First direction:} 
    Let $\tau \in \cT(\cK)$.
    In particular, there exist functions $g, g' \in \cG$ such that $K \sim_\mu K' E_\tau$ for the mixing kernels $K = \Dirac_g$ and $K' = \Dirac_{g'}$.
    It is straightforward that $K \sim_\mu K' E_\tau$ implies that $g \sim_\mu g' \circ \tau$.
    Composing both sides with $(g')^{-1}$, we obtain $(g')^{-1} \circ g \sim_\mu \tau$, i.e., $\tau \in \cT(\cG)$.

    \textbf{Second direction:} Let $\tau \in \cT(\cG)$.
    In particular, there exists functions $g, g' \in \cG$ such that $\tau \sim_\mu g^{-1} \circ g'$.
    Composing both sides with $g$, we obtain $g \circ \tau \sim_\mu g'$.
    Thus, $K' \sim_\mu K \Dirac_\tau$ for $K = E_g$ and $K' = E_{g'}$.
    Hence, $\tau \in \cT_{K'}(\cK)$, and since $\cT_{K'}(\cK) \subseteq \cT(\cK)$, we have $\tau \in \cT(\cK)$. 
\end{proof}

We note that, to drop the image equivalence assumption, one would need to define the set
\begin{equation*}
\cT(\cG) = \{ \tau : \tau \sim_\mu g^{-1} \circ g'\ \text{for some}\ g, g' \in \cG\ \text{such that}\ \supp(g_\sharp(\mu)) = \supp(g'_\sharp(\mu)) \}
\end{equation*}
and carefully handle domain issues, similar to the proof of \Cref{thm:feature-implies-blackwell}.

\section{DEFERRED PROOFS (APPLICATIONS)}

\subsection{Definition of the Canonical Stratified Measure}\label{appendix:canonical-stratified-measure}

For $J \subseteq [d]$, let $S_J \defeq \spann(\{ e_i : i \in J \})$.
For $|J| = k$, we let $\lambda_J$ denote the $k$-dimensional Lebesgue measure on $S_J$, i.e., $\lambda_J = (\iota_J)_\sharp \lambda^k$, where $\iota_J\colon \bbR^k \to S_J$ is the canonical isomorphism between $\bbR^k$ and $S_J$, and $\lambda^k$ is Lebesgue measure on $\bbR^k$.
Then, we define the \defword{canonical stratified measure} $\lambda^d_s$ referred to in \Cref{sec:dictionary-learning}  as
\begin{equation*}
    \lambda^d_s 
    \defeq 
    \sum_{\substack{J \subseteq [d] \\ |J| \leq s}} \lambda_J.
\end{equation*}

\subsection[Proofs from Section 5.1]{Proofs from \Cref{sec:dictionary-learning}}\label{proofs:dictionary-learning}

Before \Cref{claim:spark}, we cited the fact that any matrix with spark greater than $2s$ is injective on the set $\SparseVectors^d_s$.
This classical result can be found in many places, e.g. Theorem 2.4 of \citet{elad2010sparse}.
For completeness and consistency with our notation, we now formally state and prove this result.
\begin{claim}
    Let $\matG \in \bbR^{p \times d}$ with $\spark(\matG) > 2s$, and let $\tau\colon c \mapsto \matG \cdot c$.
    Then $\matG$ is injective on $\Sigma^d_s$.
\end{claim}
\begin{proof}
    Suppose $c, c' \in \Sigma^d_s$ and $\matG \cdot c = \matG \cdot c'$.
    Let $\Delta = c - c'$, so that $\matG \cdot \Delta = 0$.
    Since $\| c \|_0 \leq s$ and $\| c' \|_0 \leq s$, we have $\| \Delta \|_0 \leq 2s$.
    Thus, $\matG \cdot \Delta$ is a linear combination of at most $2s$ columns of $\matG$.
    Since $\spark(\matG) > 2s$, these columns are linearly independent, so $\matG \cdot \Delta = 0$ implies that $\Delta = 0$.
    This shows that $c = c'$, as desired.
\end{proof}

Our main claim in \Cref{sec:dictionary-learning} characterizes the valid transition set $\cT(\cK_\spark)$ for the set of mixing functions $\cK_\spark = \cK(\SparseVectors^d_s \to \bbR^p ; \LinearMix, \Spark_\sigma)$.
This characterization is the main difficulty in proving identifiability of $\cM_\spark$, and a version of this characterization appears in the proof of Theorem 3 from \citet{aharon2006uniqueness}.

Here, we hope to provide more intuition for the combinatorial nature of this proof.
For $s \leq d$, we define $\cJ(d, s)$ as the collection of all size-$s$ subsets of $[d]$, i.e.,
\begin{equation*}
    \cJ(d, s) \defeq \{ J \subseteq [d] : |J| \leq s \}.
\end{equation*}
For $i \in [d]$, we define $\cJ_i(d, s)$ as the collection of all size-$s$ subsets of $[d]$ which contain $i$, i.e.,
\begin{equation*}
    \cJ_i \defeq \{ J \in \cJ : i \in J \}.
\end{equation*}



\begin{claim}\label{claim:sparsity-preserving-implies-permutation}
    Let $\matT \in \bbR^{d \times d}$ be an invertible matrix.
    For $s \leq d - 1$, assume $\matT(\SparseVectors^d_s) \subseteq \SparseVectors^d_s$.
    Then $\matT \in \Relabelings(d) \circ \Scale(d)$.
\end{claim}
\begin{proof}
    For $i \in [d]$, let $e_i$ denote the $i$\textsuperscript{th} canonical basis vector.
    For $J \subseteq [d]$, let $S_J$ denote the coordinate-aligned subspace spanned by the corresponding basis vectors, i.e., $S_J \defeq \spann(\{ e_i : i \in J \})$, and let $W_J \defeq T[S_J]$.
    We note that, for $s \leq d-1$, any $s$-dimensional subspace $S$ of $\Sigma^d_s$ is coordinate-aligned, i.e., if $\dim(S) = s$, then $S = S_J$ for some $J \in \cJ(d, s)$.

    \textbf{Induced mapping on size-$s$ coordinate-aligned subspaces:}
    Let $J \in \cJ(d, s)$.
    Since $\matT$ is invertible, we have $\dim(W_J) = s$, and since $\matT(\Sigma^d_s) \subseteq \Sigma^d_s$, we have $W_J = S_{J'}$ for some $J' \in \cJ$.
    Let $\phi\colon \cJ(d,s) \to \cJ(d,s)$ be the map which sends $J \in \cJ(d,s)$ to $J' \in \cJ(d,s)$ such that $W_J = S_{J'}$.

    \textbf{Induced mapping on axes:} 
    Let $i \in [d]$.
    Since $\matT$ is invertible, we have $\dim(W_{\{i\}}) = 1$.
    Note that $S_{\{i\}} = \bigcap_{J \in \cJ_i(d,s)} S_J$, so by definition of the image and $\phi$, we have $W_{\{i\}} = \bigcap_{J \in \cJ_i(d,s)} S_{\phi(J)}$, i.e., $W_{\{i\}}$ must be a coordinate-aligned subspace.
    Thus, since $\dim(W_{i}) = 1$, we must have $W_{\{i\}} = S_j$ for some $j \in [d]$.

    \textbf{Conclusion:}
    The induced mapping on axes shows that, for each $i \in [d]$, there is one nonzero entry in the $i$\textsuperscript{th} column of $\matT$.
    Since $\matT$ is invertible, it must be a generalized permutation matrix, i.e., $\matT \in \Relabelings(d) \circ \Scale(d)$.
\end{proof}

Using this result, the characterization of $\cT(\cK_\spark)$ becomes straightforward:
\begin{claim*}[Restatement of \Cref{claim:spark}]
    Let $K, K' \in \cK_\spark$ defined in \Cref{eqn:m-spark}, and assume $\sigma > 2s$.
    If $K \sim_{\lambda^d_s} K' E_\tau $ for some $\tau \in \MeasInj(\SparseVectors^d_s \to \SparseVectors^d_s)$, then  $\tau \in \Relabelings(d) \circ \Scale(d)$.

    More concisely, $\cT(\cK_\spark) = \Relabelings(d) \circ \Scale(d)$.
\end{claim*}
\begin{proof}
    First, we unpack definitions.
    Since $K, K' \in \cK(\Sigma^d_s \to \bbR^p ; \LinearMix, \Spark_\sigma)$, there exist matrices $\matG, \matG' \in \bbR^{p \times d}$, with $\spark(\matG) \geq \sigma > 2s$ and $\spark(\matG') \geq \sigma > 2s$, such that $K\colon c \mapsto \dirac_{\matG \cdot c}$ and $K'\colon c \mapsto \dirac_{\matG' \cdot c}$.

    Since $\matG$ and $\matG'$ are both injective on $\SparseVectors^d_s$ and linear, and $\tau$ is injective, $\tau$ must also be linear.
    Thus, we represent $\tau$ by an invertible matrix $\matT \in \bbR^{d \times d}$ such that $\tau\colon c \mapsto \matT \cdot c$.
    Moreover, since $\tau \in \MeasInj(\SparseVectors_s^d \to \SparseVectors_s^d)$, $\matT$ must be sparsity-preserving, i.e., $\| \matT \cdot c \|_0 \leq s$ for any $c$ such that $\| c \|_0 \leq s$.

    Applying \Cref{claim:sparsity-preserving-implies-permutation} on $\matT$, we obtain the result.
\end{proof}

\subsection[Proofs from Section 5.2]{Proofs from \Cref{sec:ica}}\label{proofs:ica}

In \Cref{sec:ica}, we recounted the classic result that $\cT(\cK) \cap \cT(\cQ_\indQ^d) = O(d)$ when $\cK = \cK(\bbR^d \to \bbR^p ; \LinearMix)$, where $O(d)$ denotes the orthogonal group on $\bbR^d$.
For completeness, we give a brief proof.
\begin{claim}
    For $d \leq p$, let $\cQ^d_\indQ = \cP(\bbR^d ; \Ind, \UnitVariance)$ and let $\cK = \cK(\bbR^d \to \bbR^p; \LinearInj)$.
    Then $\cT(\cK) \cap \cT(\cQ^d_\indQ) = O(d)$.
\end{claim}
\begin{proof}
    First, we show that $\cT(\cK) \cap \cT(\cQ_\indQ) \subseteq O(d)$.
    Recall that $\cT(\cK) = \GL(d)$.

    Let $\tau\colon \bc \to \matT \cdot \bc$ for some invertible matrix $\matT \in \bbR^{d \times d}$, so that $\tau \in \cT(\cK)$.
    Further, assume $Q, Q' \in \cQ_\indQ$ with $Q' = \tau_\sharp(Q)$, so that $\tau \in \cT(\cQ_\indQ)$.

    By definition of the $\UnitVariance$ concept predicate, we have $\Cov(Q) = \Cov(Q') = \matI_d$.
    By linearity of expectation, we have $\Cov(Q') = \matT \cdot \Cov(Q) \cdot \matT^\top$.
    Thus, we obtain $\matT \cdot \matT^\top = \matI_d$, i.e., $\matT \in O(d)$.

    Now, we show that $O(d) \subseteq \cT(\cK) \cap \cT(\cQ_\indQ)$.
    Let $\tau \in O(d) \subseteq \cT(\cK)$.
    Then, for $Q = Q' = \cN(\bzero, \matI_d)$, we have $Q' = \tau_\sharp(Q)$, i.e., $\tau \in \cT(\cQ_\indQ)$, completing the proof.
\end{proof}

Our main claim in \Cref{sec:ica} is a minor adaptation of Theorem 11 in \citet{comon1994independent}, which shows that any orthogonal transformation of a product distribution with non-Gaussian components will no longer be a product distribution, unless the transformation is a generalized permutation.\footnote{As will be evident from the proof, one Gaussian component is allowed, and deterministic components need to be ruled out; these are built into the definition of the $\NonGaussian$ mixing predicate.}
The proof is a corollary of Darmois's theorem, which we restate here for convenience:
\begin{thm*}[Darmois's theorem]
    Let $(C_j)_{j=1}^d$ be a collection of independent random variables.
    Let $C'_1 = \sum_{j=1}^d a_j \cdot C_j$, and let $C'_2 = \sum_{j=1}^d b_j \cdot C_j$.
    Finally, let $\cI = \{ j \in [d] : a_j \cdot b_j \neq 0 \}$.

    If $C'_1$ and $C'_2$ are independent, then all $(C_j)_{j \in \cI}$ have Gaussian distributions (possibly with zero variance).
\end{thm*}
As a corollary of Darmois's theorem, we obtain the following:
\begin{corollary}\label{cor:darmois}
    Let $Q, Q' \in \cP(\bbR^d ; \Ind)$.
    Suppose $Q' = \tau_\sharp(Q)$ for $\tau\colon c \mapsto \matT \cdot c$, where $\matT$ has at least two nonzero entries in the $k$\textsuperscript{th} column, i.e., $\matT_{uk} \neq 0$ and $\matT_{vk} \neq 0$ for $u \neq v$.
    Then the $k$\textsuperscript{th} component of $Q$ has a Gaussian distribution.
\end{corollary}
\begin{proof}
    Define random variables $C \sim Q$ and let $C' = \matT \cdot C$, so that $C' \sim Q'$.
    We have $C'_u = \sum_{j=1}^d \matT_{uj} \cdot C_j$ and $C'_v = \sum_{j=1}^d \matT_{vj} \cdot C_j$.
    By the assumption that $Q' \in \cP(\bbR^d ; \Ind)$, $C'_u$ and $C'_v$ are independent.
    By the premise of the corollary, $\matT_{uk} \cdot \matT_{vk} \neq 0$.
    Hence, by Darmois's theorem, $C_k$ has a Gaussian distribution.
\end{proof}

Applying this corollary and using properties of orthogonal matrices, we obtain the main result:
\begin{claim*}[Restatement of \Cref{claim:nongaussianity}]
    $\cT(\cQ^d_\indng) \cap O(d) = \Relabelings(d)$.
\end{claim*}
\begin{proof}
    The direction $\Relabelings(d) \subseteq \cT(\cQ_\indng) \cap O(d)$ is straightforward, since non-Gaussianity and independence are unaffected by relabeling.

    To show that $\cT(\cQ_\indng) \cap O(d) \subseteq \Relabelings(d)$, let $\tau\colon c \mapsto \matT \cdot c$ for some orthogonal matrix $\matT$ which is not a permutation matrix.
    Then there must be at least two columns $j, k \in [d]$ such that $\matT$ has at least two nonzero entries in each of these columns.
    Next, suppose that $Q' = E_\tau Q$ for $Q, Q' \in \cP(\bbR^d ; \Ind)$.
    By \Cref{cor:darmois}, the $j$\textsuperscript{th} and $k$\textsuperscript{th} component of $Q$ must be Gaussian.
    Since this applies for any $Q$ and $Q'$, we have $\tau \not\in \cT(\cQ_\indng)$.
    Thus, $\cT(\cQ_\indng) \cap O(d) \subseteq \Relabelings(d)$.
\end{proof}


\section{METHODOLOGICAL IMPLICATIONS FOR SPARSE AUTOENCODERS}\label{appendix:saes-methods}

\subsection{Post-hoc Checks}\label{appendix:post-hoc-checks}

Compared to the model class $\cM_\spark(d, p ; s, \sigma)$, the model class $\cM_\sae(d, p ; s)$ lacks the spark constraint over the mixing kernel $K$.
As discussed in \Cref{sec:dictionary-learning}, the spark constraint ensures injectivity of the mixing kernel, which is required for the class of mixing kernels $\cK_\spark$ to be Blackwell reducible, and thus for the proof of \Cref{claim:dictionary-learning-id}.
This raises an important question: if one learns a model $\M \in \cM_\sae(d, p ; s)$, what can be concluded about the identifiability of this model?

As a partial answer to this question, we suggest performing a post-hoc check to determine whether $\M \in \cM_\spark(d, p ; s, \sigma)$.
In particular, after training a sparse autoencoder, one may check whether the matrix $\matG$ associated with its decoder obeys the desired spark constraint, i.e., whether $\spark(\matG) > 2s$.
If this condition holds, then one can make the following conclusion from \Cref{claim:dictionary-learning-id}: the model $\M$ is identifiable \emph{within} the model class $\cM_\spark(d, p ; s, \sigma)$.
This conclusion reflects a generally useful strategy: for computational reasons, one may avoid enforcing some constraints during optimization (e.g. the spark constraint), and instead rely on post-hoc checks to determine identifiability guarantees.
However, one must be mathematically careful in such cases: identifiability guarantees must be stated \emph{with respect to a particular model class} (\textit{c.f.} the definition of identifiability in \Cref{sec:intersection}).
Finally, we note that for such checks, computing the spark of a matrix is NP-hard, but there are computationally tractable lower bounds (e.g. based on mutual incoherence) \citep{elad2010sparse}.

\subsection{Amortized Posterior Inference}\label{appendix:amortized-posterior-inference}

When concept extraction needs to be performed efficiently, the posterior inference methods described in \Cref{sec:concept-extraction} may be too slow, e.g. even for a deterministic decoder $g_\psi\colon c \mapsto \matG_\psi \cdot c$, exactly computing the inverse $g_\psi^{-1}(z)$ via matching pursuit may be too expensive.
Hence, one may instead wish to learn a parameterized concept extractor (i.e., an encoder) $H_\phi$ for some parameters $\phi \in \Phi$, e.g. by training on a loss function that enforces $H_\phi \approx H_{\theta^*}$.

Here, we discuss two options for such amortization: \emph{post-hoc amortization}, where $\phi$ is optimized after training the generative model $\M_{\thetafix} = (Q_\omegafix, K_\psifix)$, and \emph{end-to-end amortization}, where $\phi$ is optimized at the same time as $\theta = (\omega, \psi)$.
As a reminder, for a model $\M_\theta = (Q_\omega, K_\psi)$, we let $J_\theta(c, z) = K_\psi(z \mid c) \cdot Q_\omega(c)$ denote its associated joint distribution over concepts and features, and we let $P_\theta(z) = \int K_\psi(z \mid dc) \cdot Q_\omega(dc)$ denote its associated marginal distribution over features.
Further, we let $H_\theta(c \mid z) = \frac{K_\psi(z \mid c) \cdot Q_\omega(c)}{P_\theta(z)}$ denote the induced posterior.
Finally, we will see that many approaches can be interpreted as performing an amortized version of \emph{maximum a posteriori} (MAP) inference, so we define the \defword{induced MAP function} $\barh_\theta\colon \cZ \to \cC$ as follows:
\begin{equation*}
    \barh_\theta(z) 
    \defeq 
    \argmax_{c \in \cC} H_\theta(c \mid z)
    =
    \argmax_{c \in \cC} \log K_\psi(z \mid c) + \log Q_\omega(c),
\end{equation*}
where for simplicity we assume there is a unique maximizer for all $z \in \cZ$.

\paragraph{SAE Running Example}
To connect the maximum likelihood perspective (common in generative modeling) with the regularized reconstruction perspective (common for sparse autoencoders), we follow \citet{geadah2024sparse}, who (in our terminology) consider mixing kernels of the form $K_\psi(\cdot \mid c) = \cN(\matG_\psi \cdot c, \sigma^2 I)$, and who consider a fixed concept distribution $Q$, e.g. a product of Laplace distributions $Q(c_i) \propto \exp(-\alpha |c_i|)$ over each component $i$ of $\rvC$.
Under these choices, one obtains
\begin{equation}\label{eqn:geadah-likelihoods}
    \log K_\psi(z \mid c) \simeq -\frac{1}{2 \sigma^2} \| z - \matG_\psi \cdot c \|_2^2
    \quad\text{and}\quad 
    \log Q(c) \simeq -\alpha \| c \|_1,
\end{equation}
where $\simeq$ denotes equality up to an absolute constant.

\subsubsection{Post-hoc amortization}
Assuming evaluation access to $H_\phi$, a natural choice is the cross entropy loss
\begin{equation*}
    \cL_{\text{CE}}(\phi ; \thetafix)
    \defeq
    \bbE_{(\rvC, \rvZ) \sim J_\thetafix} \left[
        -\log H_\phi(\rvC \mid \rvZ)
    \right],
\end{equation*}
which by well-known theory is optimized when $H_\phi(c \mid z) = H_\thetafix(c \mid z)$ almost everywhere on $P_\thetafix$ (assuming $\{ H_\phi \}_{\phi \in \Phi}$ has universal approximation capability).
This approach would resemble recent \emph{neural posterior estimation} approaches in simulation-based inference \citep{lueckmann2017flexible,deistler2025simulation} and \emph{inference compilation} approaches in probabilistic programming \citep{le2017inference}, with roots in the ``wake-sleep" algorithm \citep{hinton1995wake}.
As an illuminating example, one may consider optimizing over concept extractors of the form $H_\phi(\cdot \mid z) = \cN(h_\phi(z), \sigma^2 I)$, giving
\begin{equation*}
    \cL_{\text{CE}}(\phi ; \thetafix)
    \simeq
    \bbE_{(\rvC, \rvZ) \sim J_\thetafix} \left[
        -\frac{1}{2\sigma^2} \| \rvC - h_\phi(\rvZ) \|_2^2
    \right],
\end{equation*}
a least-squares objective for reconstructing the sampled concept $\rvC$ which generated the associated $\rvZ$.

\paragraph{Amortized MAP inference}
As described in \citet{geadah2024sparse}, the use of a \emph{deterministic} encoder can be thought of as performing maximum a posteriori (MAP) inference.
Hence, to learn an amortized MAP function $h_\lambda$ (with the aim that $h_\phi \approx \barh_\theta)$, one may minimize the objective
\begin{equation*}
    \cL_{\text{MAP}}(\phi ; \thetafix)
    \defeq
    \bbE_{\rvZ \sim P^*} \left[
        -\log K_\psifix(\rvZ \mid h_\phi(\rvZ))
        - \log Q_\omegafix(h_\phi(\rvZ))
    \right].
\end{equation*}
Thus, assuming the log-likelihoods in \Cref{eqn:geadah-likelihoods}, one obtains
\begin{equation*}
    \cL_{\text{MAP}}(\phi ; \thetafix)
    \simeq
    \bbE_{\rvZ \sim P^*} \left[ 
    \frac{1}{2 \sigma^2} \| \rvZ - \matG_\psifix \cdot h_\phi(\rvZ) \|_2^2
    + \alpha \| h_\phi(\rvZ) \|_1
    \right],
\end{equation*}
which is precisely the standard SAE objective, but with the decoder parameters $\psi$ frozen to $\psi = \psifix$.

\subsubsection{End-to-end amortization}
As a final option, the concept extractor $H_\phi$ could be trained in parallel with the generative model $\M_\theta$.
When the prior is learnable, such end-to-end training could be performed using a general form of the evidence lower bound (ELBO)\footnote{For consistency with other loss functions, we present the ELBO as an upper bound on the \emph{negative} log likelihoood.}, as found in \citet{tomczak2018vae}:
\begin{align*}
    \bbE_{\rvZ \sim P^*} [-\log P^*(\rvZ)]
    &\leq 
    \cL_{\text{GELBO}}(\phi, \psi, \omega),\ \text{where}
    \\
    \cL_{\text{GELBO}}(\phi, \psi, \omega)
    &\defeq
    \bbE_{\rvZ \sim P^*, \rvC \sim H_\psi(\cdot \mid \rvZ)} \left[ 
    -\log K_\psi(\rvZ \mid \rvC)
    -\log Q_\omega(\rvC)
    +\log H_\phi(\rvC \mid \rvZ)
    \right]
\end{align*}
When the prior is fixed to $Q$ (e.g. a Laplace distribution as describe above), this bound reduces to the more standard ELBO:
\begin{equation*}
    \cL_{\text{ELBO}}(\phi, \psi) \defeq 
    \bbE_{\rvZ \sim P^*, \rvC \sim H_\psi(\cdot \mid \rvZ)} \left[ 
    \log K_\psi(\rvZ \mid \rvC) + \log \frac{Q(\rvC)}{H_\phi(\rvC \mid \rvZ)}
    \right].
\end{equation*}
To connect our framework with traditional SAEs (which use deterministic encoders), one may take the regularization perspective described in \citet{ghosh2019variational}.
Alternatively, to continue with the maximum likelihood perspective, one may consider the \defword{deterministic variational family} $H_\phi = \Dirac_{h_\phi}$.
Formally, this choice cannot be interpreted as a density, but can be considered as a certain limit of the variational distributions $H_\phi^\beta(c \mid z) = \cN(h_\phi(z), \beta^{-1} I)$ as $\beta \to \infty$\footnote{This reduction can also be seen more intuitively, though less rigorously: under $H_\phi = E_{h_\phi}$, the expectation is only evaluated at $\rvC = h_\phi(\rvZ)$, where $H_\phi(\rvC \mid \rvZ) = \infty$ and the second term vanishes.}, yielding
\begin{equation*}
    \cL_{\text{ELBO-lim}}(\phi, \psi)
    \defeq
    \bbE_{\rvZ \sim P^*} \left[
    -\log K_\psi(\rvZ \mid h_\phi(\rvZ))
    \right].
\end{equation*}
Under the likelihood in \Cref{eqn:geadah-likelihoods}, we thus have
\begin{equation*}
    \cL_{\text{ELBO-lim}}(\phi, \psi)
    \simeq
    \bbE_{\rvZ \sim P^*} \left[
        \frac{1}{2\sigma^2} \| \rvZ - \matG_\psi \cdot h_\phi(\rvZ) \|_2^2
    \right],
\end{equation*}
the classic reconstruction loss for SAEs.
Hence, the final distinction between existing SAE approaches is their choice of the deterministic variational family: many SAEs restrict to \emph{projection encoders} (as described below), whereas matching pursuit SAE \citep{costa2025flat} and other methods \citep{tolooshams2021stable,xiao2025sc} use variational family which more closely resemble classical methods for computing $g_\psi$, such as matching pursuit or the iterative shrinkage and thresholding algorithm \citep{daubechies2004iterative,gregor2010learning}.

\subsection{Comparison to the Projective Assumptions Framework}\label{appendix:hindupur-comparison}

\citet{hindupur2025projecting} discuss several SAE architectures, including ReLU SAEs \citep{cunningham2023sparse}, TopK SAEs \citep{makhzani2013k}, JumpReLU SAEs \citep{rajamanoharan2024jumping,lieberum2024gemma}, and their newly-introduced SparseMax SAEs.
They show that each of these architectures uses a \defword{constraint set} $\cS \subseteq \bbR^d$ and a \defword{projection encoder} $h\colon \cZ \to \cS$ of the form
\begin{equation}\label{eqn:projection-encoder}
    h_\phi(z) = \Pi_\cS(\matW_\phi \cdot z + b_\phi)
\end{equation}
for some matrix $\matW_\phi \in \bbR^{d \times p}$ and some vector $b_\phi \in \bbR^{d}$, where $\Pi_\cS$ denotes the projection operator onto $\cS$.\footnote{
Strictly speaking, the projection operator $\Pi_\cS(v) = \argmin_{v' \in \cS} \| v - v' \|_2^2$ is only well-defined when $\cS$ is closed and convex, so that the minimizer exists and is unique.
This point was not explicitly mentioned in \citet{hindupur2025projecting}, but applies to the constraint sets they considered.
}
As noted in \citet{hindupur2025projecting}, the form of $h_\phi$ entails implicit assumptions over (1) the geometry of the concept space $\cC$ and (2) the functional relationship between concepts and features.
Our framework proposes to make these assumptions more transparent by (1) explicitly defining the concept space $\cC$ and (2) explicitly stating assumptions on the model class $\cM$.

One straightforward implication of assuming a projection encoder is that the concept space $\cC$ is constrained by the constraint set $\cS$ of the projection operator.
For example, since ReLU SAEs project onto the nonnegative orthant (i.e., $\ReLU(v) = \Pi_{\cS}(v)$ for $\cS = \bbR^d_{\geq 0}$), they impose $\cC \subseteq \bbR^d_{\geq 0}$; since TopK SAEs project onto the the set of nonnegative sparse vectors, they impose $\cC \subseteq \bbR^d_{\geq 0} \cap \SparseVectors^d_s$; and so on, as summarized in \Cref{tab:explicit-constraints}.

\begin{table}[h]
    \centering
    \begin{tabular}{c|c}
         Model & $\cS$ \\
         \hline
         ReLU SAE & $\bbR^d_{\geq 0}$ \\
         TopK & $\bbR^d_{\geq 0} \cap \SparseVectors^d_s$ \\
         Heaviside & $\{ 0, 1 \}^d$  \\
         JumpReLU & $\bbR^d_{\geq 0}$  \\
         SparseMax & $\Simplex^{d-1}$ 
    \end{tabular}
    \caption{Constraint sets used by different SAE architectures.}
    \label{tab:explicit-constraints}
\end{table}

However, the assumption of a projection encoder imposes a less straightforward assumption on the model class $\cM$.
In general, this assumption cannot be decomposed into separate assumptions on the concept distribution $Q$ and the mixing kernel $K$.
Instead, the implicit model class used by an SAE with a linear decoder and the projection encoder in \Cref{eqn:projection-encoder} is as follows:
\begin{equation*}
    \cM_\projenc(d, p ; \cS) = \cM_\all(\cC, \cZ ; \ProjEnc_\cS),
\end{equation*}
where $\ProjEnc_\cS\colon \cM_\all(\cC, \cZ) \to \{ \True, \False \}$ is a \emph{model predicate}, which is $\True$ for $\M = (Q, K)$ if and only if the following conditions hold:
\begin{enumerate}
    \item $K\colon c \mapsto \matG \cdot c$ for some matrix $\matG \in \bbR^{p \times d}$.
    \item $\matG$ is injective on $\supp(Q)$, with a left-inverse of the form $h(z) = \Pi_\cC(\matW \cdot z + b)$
\end{enumerate}
In particular, the second condition involves \emph{both} the concept distribution $Q$ and the mixing kernel $K$.
Due to this lack of product structure on the model class, \Cref{thm:intersection} cannot be applied to these settings, and we propose that the use of projection encoders should be treated as an \emph{approximation} to the true left-inverse $h$ (as discussed in \Cref{sec:methodological}, rather than as additional assumption.
More generally, this discrepancy reflects well-known dichotomies, e.g. the dichotomy between \emph{generative} and \emph{discrimative} models \citep{ng2001discriminative}, and the dichotomy between causal and anticausal learning \citep{kugelgen2020semi}.
Indeed, in the general case of a stochastic mixing kernel (decoder), we have
\begin{equation*}
    H^\M(c \mid z) = \frac{K(z \mid c) \cdot Q(c)}{\int_\cC K(z \mid dc) \cdot Q(c)}
\end{equation*}
by Bayes' rule, indicating that assumptions directly on $H^\M$ should typically be expected to involve assumptions on \emph{both} $Q$ and $K$.
Given our focus on the generative perspective, we do not explore this issue further here, but note it as an interesting direction for future elaboration upon the framework presented here.

\section{ADDITIONAL APPLICATION: FINITE MIXTURE MODELS}\label{appendix:finite-mixture-models}

As a final example, we consider a model class with stochastic mixing functions.
Broadly speaking, a \emph{finite mixture model} is a latent generative concept model where $\cC = [d]$.
We endow $\cC$ with the standard counting measure $\mu$, and we employ the concept predicate $\MeasureClass_\mu$, so that $\MeasureClass_\mu(Q)$ holds if and only if $Q(\{i\}) > 0$ for all $i \in [d]$.
Other than this constraint, we make no restrictions on the concept distribution.

For simplicity, we consider the case of \emph{Gaussian} mixture models.
In particular, we define the mixing predicate $\DistinctGaussian$, where $\DistinctGaussian(K)$ holds if and only $K(\cdot \mid c)$ is a multivariate Gaussian distribution for all $c \in \cC$, and $K(\cdot \mid c) \neq K(\cdot \mid c')$ for any $c \neq c'$.

Putting these constraints together, we obtain the model class
\begin{equation*}
    \cM_\gmm(d, p)
    \defeq
    \cP([d] ; \MeasureClass_\mu)
    \times
    \cK([d] \to \bbR^p ; \DistinctGaussian).
\end{equation*}

As discussed in \Cref{example:br-gaussian-mixture}, $\cK([d] \to \bbR^p ; \DistinctGaussian)$ is Blackwell reducible, so we can apply \Cref{cor:intersection-class}, which gives the result almost immediately.
Here, we use $\sim_\perm$ as shorthand for the equivalence relation $\sim_G$ with $G = \Permutations(d)$, where $\Permutations(d)$ denotes the set of all bijective functions $\tau\colon [d] \to [d]$.

\begin{claim}
    $\cM_\gmm$ is identifiable up to $\sim_\perm$.
\end{claim}
\begin{proof}
    Let $\cQ = \cP([d] ; \MeasureClass_\mu)$.
    Then $\cT(\cQ) = \Permutations(d)$, so we immediately have the result.
\end{proof}

As this result shows, the main obstacle to showing identifiability of certain classes of finite mixture models is showing that the kernel class $\cK$ is Blackwell reducible.
Indeed, showing this property is the essence of several other classic identifiability results for finite mixture models, e.g. \citet{yakowitz1968identifiability}.

\printbibliography[heading=mybibhead,title={Additional References}]
\end{refsection}

\end{document}